\documentclass{article}

\PassOptionsToPackage{numbers, compress}{natbib}

\usepackage[final]{neurips_2022}

\usepackage[utf8]{inputenc} 
\usepackage[T1]{fontenc}    
\usepackage{hyperref}       
\usepackage{url}            
\usepackage{booktabs}       
\usepackage{amsfonts}       
\usepackage{nicefrac}       
\usepackage{microtype}      
\usepackage{xcolor}         

\usepackage{amsmath,amsthm,amssymb,dsfont,bm}

\usepackage{mathrsfs}

\usepackage{booktabs}

\newcommand{\E}[1]{\mathbb{E}\left[#1\right]}

\newcommand{\limm}{\lim\limits_{m\rightarrow\infty}}
\newcommand{\scp}[2]{\langle #1,#2\rangle}

\newcommand{\xn}[1]{x^{(m)}_{#1}}

\def\R{\mathbb{R}}

\def\var{\text{-var}}

\newcommand{\ds}{\mathrm{d}}

\usepackage{tikz}
\usepackage{pgfplots}
\usepackage{mathtools}
\usepackage{wrapfig}

\usepackage{enumitem}

\newcommand{\RR}{\mathbb{R}}  
\newtheorem{theorem}{Theorem}[section]
\newtheorem{lemma}{Lemma}[section]
\newtheorem{proposition}{Proposition}[section]

\theoremstyle{definition}
\newtheorem{definition}{Definition}[section]

\newtheorem{remark}{Remark}[section]
\newtheorem{ass}{Assumption}

\title{Chaotic Regularization and Heavy-Tailed Limits for Deterministic Gradient Descent}

\author{%
  Soon Hoe Lim \\
  Nordita\\
  KTH Royal Institute of Technology 
  and Stockholm University \\
  \texttt{soon.hoe.lim@su.se} \\
   \And
   Yijun Wan \\
   D\'epartement de Math\'ematiques et Applications \\ 
   \'{E}cole Normale Sup\'{e}rieure, Universit\'e PSL \\
   \texttt{wan@clipper.ens.fr} \\
   \AND
   Umut \c{S}im\c{s}ekli \\
   DI ENS, \'{E}cole Normale Sup\'{e}rieure, Universit\'e PSL, CNRS, INRIA \\
   \texttt{umut.simsekli@inria.fr}  
 }

\begin{document}

\maketitle

\begin{abstract}

%



Recent studies have shown that gradient descent (GD) can achieve improved generalization when its dynamics exhibits a chaotic behavior. However, to obtain the desired effect, the step-size should be chosen sufficiently large, a task which is problem dependent and can be difficult in practice. In this study, we incorporate a chaotic component to GD in a controlled manner, and introduce \emph{multiscale perturbed GD} (MPGD), a novel optimization framework where the GD recursion is augmented with chaotic perturbations that evolve via an independent dynamical system. We analyze MPGD from three different angles: (i) By building up on recent advances in rough paths theory, we show that, under appropriate assumptions, as the step-size decreases, the MPGD recursion converges weakly to a stochastic differential equation (SDE) driven by a heavy-tailed L\'{e}vy-stable process. (ii) By making connections to recently developed generalization bounds for heavy-tailed processes, we  derive a generalization bound for the limiting SDE and relate the worst-case generalization error over the trajectories of the process to the parameters of MPGD. (iii) We analyze the implicit regularization effect brought by the dynamical regularization and show that, in the weak perturbation regime, MPGD introduces terms that penalize the Hessian of the loss function. Empirical results are provided to demonstrate the advantages of MPGD.

\end{abstract}


\section{Introduction}
\label{sect_intro}

Many important problems in supervised learning can be expressed by the following population risk minimization problem:
\begin{align}\label{eq:pop-risk}
\min_{x\in\mathbb{R}^d} \Bigl\{ \mathcal{R}(x) := \mathbb{E}_{z\sim\mathcal{D}} [\ell(x,z)] \Bigr\},
\end{align}
where $x \in \mathbb{R}^d$ denotes the learnable parameter, $z \in \mathcal{Z}$ denotes a data sample coming from an unknown data distribution $\mathcal{D}$, $\mathcal{Z}$ denotes the space of data points, and $\ell : \mathbb{R}^d \times \mathcal{Z} \to \mathbb{R}_+$ is the loss function, measuring the quality of the parameter $x$.
Since the data distribution $\mathcal{D}$ is not known, as a proxy for solving \eqref{eq:pop-risk}, we consider the empirical risk minimization (ERM) problem, given as follows:
\begin{align}\label{eq:erm}
\min_{x\in\mathbb{R}^d} \Bigl\{ \hat{\mathcal{R}}(x, S_n) := \frac1{n} \sum_{i=1}^n \ell(x,z_i) \Bigr\},
\end{align}
where $S_n := \{z_1,\dots,z_n\}$ denotes a \emph{training set} of $n$ points that are independently and identically distributed (i.i.d.) and sampled from $\mathcal{D}$. To solve the ERM problem, typically gradient-based optimization algorithms are being used in practice. 

In this study, we will consider the gradient descent (GD) algorithm \cite{cauchy1847methode}, which is based on the following deterministic recursion:
\begin{align}\label{eqn:sgd}
  x_{k+1} = x_{k} - \eta_k \nabla \hat{\mathcal{R}} (x_{k}, S_n), \ \ \ k=0,1,2,\dots.
\end{align}
Here, $k$ represents the iteration number, $\eta_k > 0$ is the step-size (learning rate), and $x_0 \in \RR^d$ is the initialization. Even though stochastic counterparts of GD (e.g., stochastic GD \cite{bottou2018optimization}) have been more popular due to their reduced computational requirements, the GD algorithm has attracted increasing attention in the past few years, as it has been illustrated that GD can achieve similar generalization performance, as measured by the \emph{generalization error} $|\hat{\mathcal{R}}(x_k,S_n) -\mathcal{R}(x_k)|$, as long as certain design choices are implemented \cite{geiping2021stochastic}. These design choices mainly include using large and/or oscillating learning rates and using explicit regularization of the empirical risk. 

Contrary to the wisdom of  optimization theory, which typically suggests using smaller learning rates, the fact that large learning rates often provide better generalization performance in modern non-convex optimization problems has opened up several research directions. In this context, \cite{cohen2021gradient} showed that large learning rates drive the GD dynamics towards the `edge of stability', meaning that the learning rate should be chosen large enough to make the algorithm diverge. Recently, the large learning rates have been extended to exotic learning rate schedules (i.e., non-trivial evolution of $\eta_k$), such as fractal scheduling \cite{agarwal2021acceleration} and cyclic scheduling \cite{oymak2021provable, smith2017cyclical}. 

It has been shown that the use of large learning rates introduces a \emph{chaotic behavior} in the GD dynamics \cite{kong2020stochasticity}. Hence, a possible explanation for the improvements brought by large learning rates can be partially attributed to such chaotic behavior, since chaotic systems exhibit a stochastic-like evolutions (even though they are fully deterministic), which might be an underlying cause of the improvements in generalization obtained in non-convex settings. However, achieving the desired chaotic dynamics is a problem-dependent task and can be often difficult in practice. 




In this paper, we develop a novel optimization framework, coined {\it multiscale perturbed gradient descent} (MPGD), where  a chaotic component is introduced to GD in a controlled manner. Our approach is based on extending the GD recursion via generic regularization terms, where the regularization coefficients are modulated by using external deterministic dynamical systems that exhibit chaotic characteristics. Our main contributions are  as follows.
\begin{itemize}[itemsep=0pt,topsep=0pt,leftmargin=*,align=left]
    \item We build up on recent advances in rough paths theory and homogenization techniques \cite{gottwald2013homogenization, chevyrev2016multiscale, chevyrev2020superdiffusive}, and show that, under appropriate assumptions, as the step-size decreases, the MPGD recursion converges weakly to a stochastic differential equation (SDE)  driven by a heavy-tailed L\'{e}vy-stable process (Theorem~\ref{thm:main}; see also Theorem~\ref{thm:main_app} in Appendix).
    \item We draw connections to the recent theoretical links between heavy-tailed processes and generalization \cite{simsekli2019tail, csimcsekli2020hausdorff,  hodgkinson2021generalization}, and under certain topological regularity assumptions,  we derive a generalization bound (Theorem~\ref{thm:gen}) for the limiting SDE and relate the worst-case generalization error over its trajectories to the parameters of MPGD.
    \item We show that MPGD exhibits a natural form of implicit regularization by deriving an appropriate explicit regularizer in the weak perturbation regime (Theorem~\ref{thm_implreg}). Our analysis shows that, in this regime, the derived regularizer introduces terms that penalize the Hessian of the loss function.
    \item We empirically demonstrate the advantages of MPGD on different settings and provide further support for the developed framework and the theory (see Section \ref{sect_empiricalresults}). 
\end{itemize}

We emphasize that our primary goal is to understand the behavior of deterministic GD in a chaotic setting. In this respect, we shall underline that, at this stage, our goal is not to develop a competitive algorithm that would outperform SGD, but rather to provide a solid theoretical analysis for how stochastic behavior can emerge from GD with suitable deterministic components, as well as their implication in terms of generalization error.





\section{Preliminaries and Technical Background} \label{sec:background}

\textbf{Notation. }
For a vector $v \in \RR^d$, we denote its $i$th component by either $v^i$ or $[v]^i$ and the Euclidean norm by $|v|$ (or $\|v\|$). $\rm{diag}(v)$ denotes the diagonal matrix with $i$th diagonal entry of $v^i$.  $\langle \cdot,\cdot \rangle$ denotes the dot product of two vectors.  $\mathcal{U}(d)$ denotes uniform distribution on $d$ dimension, $tr$  denotes trace, and the superscript $^T$ denotes  transposition. $\circ$ and $\odot$ denote composition and Hadamard product respectively. $\nabla$ and $\nabla^2$ denote gradient and Hessian respectively. $\mathbb{E}$ denotes expectation. $\overset{(\mathrm{d})}{=}$ and $\overset{\mathrm{(d)}}{\to}$ denote equivalence and convergence in the sense of distribution respectively. $I$ denotes the identity matrix, and $1_A$ denotes indicator function of the set $A$.

\textbf{Stable distributions.}
Let $\alpha \in (0,2]$ be a stability parameter and $d\ge 1$. A random variable $X \in \mathbb{R}^d$ is \emph{$\alpha$-stable} if $X_1,\, X_2,\,\ldots$ are independent copies of $X$, then
$n^{-1/\alpha}\sum_{j=1}^n X_j \overset{(\mathrm{d})}{=} X \text{ for all } n\ge 1$ \cite{samoradnitsky2017stable}.
Stable distributions appear as the limiting distribution in the generalized
central limit theorem (CLT) \cite{gnedenko_kolmogorov}. The case $\alpha=2$ corresponds to the Gaussian distribution, while $\alpha=1$ corresponds to the Cauchy distribution. The $p$-th moment of a stable random variable is finite if and only if $p<\alpha$. Here we are interested in the case $\alpha\in(1,2)$, such that $\mathbb{E}[|X|]< \infty$ and $\mathbb{E}[|X|^2] = \infty$. If $X$ is symmetric, then there exists a scale parameter $c>0$ and the characteristic function of $X$ is given by $
\E{\exp(i \scp{\xi}{X})} = \exp(-|c\xi|^\alpha)$ for all $\xi\in\R^d$. If $X$ is a random vector with independent components, then there exists a scale vector~$\mathbf{c}=(c^1,\ldots,c^d) \in \mathbb{R}^d_+$ and the characteristic function of $X$ is given by~$\E{\exp(i \scp{\xi}{X})} = \exp(-\sum_{j=1}^d|c^j\xi^j|^\alpha)$.

\textbf{L\'evy processes. }
L\'evy processes are stochastic processes $(\mathrm{L}_t)_{t \geq 0}$ with independent and stationary increments, and  are defined as follows \cite{applebaum2009levy}:
\begin{enumerate}[itemsep=0pt,topsep=0pt,leftmargin=*,align=left]
    \item For $N \in \mathbb{N}$ and $t_0 < t_1 < \ldots < t_N$, the increments $(\mathrm{L}_{t_i} - \mathrm{L}_{t_{i-1}})$ are independent for all $i$.
    \item For any $t > s > 0$, $(\mathrm{L}_t - \mathrm{L}_s)$ and $\mathrm{L}_{t-s}$ have the same distribution.
    \item $\mathrm{L}_t$ is continuous in probability, i.e., for all $\delta>0$ and $s \ge 0$, $\mathbb{P}(|\mathrm{L}_t - \mathrm{L}_s| > \delta) \to 0$ as $t \to s$.
\end{enumerate}
Typical examples of L\'evy processes include Brownian motion and the $\alpha$-stable processes. By the L\'evy-Khintchine formula, a L\'evy process $(\mathrm{L}_t)_{t\ge 0}$ with $\mathrm{L}_0=0$ is determined by a triplet~$(b,\Sigma,\nu)$ for some $b\in \R^d$, $\Sigma \in \R^{d\times d}$ positive semi-definite and a measure $\nu$ on $\R^d\setminus\{0\}$ such that~$\int_{x \neq 0} \min\{1, |x|^2\} \nu(dx) < \infty$. The characteristic function of $\mathrm{L}_t$ is therefore $\exp(-t\Psi(\xi))$, with the characteristic exponent $\Psi: \mathbb{R}^d \to \mathbb{C}$ defined by
\begin{equation} \label{eq_charexp}
\Psi(\xi) = - i\scp{b}{\xi} + \frac{1}{2}\scp{\xi}{\Sigma \xi} + \int_{\R^d}\left[1 - e^{i\scp{x}{\xi}} + i \scp{x}{\xi} 1_{|x|<1}\right] \nu(dx).
\end{equation}
In the L\'evy–It\^o decomposition of a L\'evy process, $b$ denotes a constant drift, $\Sigma$ is the covariance matrix of a Brownian motion and $\nu$ characterises a L\'evy jump process, which is modified to be c\`adl\`ag (right-continuous and has left limits everywhere) with countably many jumps. In particular, if $b = 0_{\R^d}$, $\Sigma = 0_{\R^{d\times d}}$ and $\nu(dx)= |x|^{-\alpha-1}dx$ for some $\alpha\in (0,2)$, then the triplet~$(0_{\R^d},0_{\R^{d\times d}}, |x|^{-\alpha-1}dx)$ gives a symmetric $\alpha$-stable process, denoted by $(\mathrm{L}^{\alpha}_t)_{t\ge 0}$.

\textbf{Marcus differential equations.}
Passing to the limit of the driving processes with jumps in an SDE is not trivial, as one needs to be careful with the meaning of integration \cite{chechkin2014marcus}. Unfortunately, the common It\^o and Stratonovich integration (which are based on Riemann-Stieltjes sums with the integrand evaluated at the left end point and the midpoint of the partition intervals respectively) fail to provide the desired convergence. To this end, we resort to Marcus differential equations \cite{marcus1981modeling}, which in their basic form are given by:
\[
\ds X_t = b(X_t) \diamond \ds \mathrm{L}_t,
\]
where $\diamond$ denotes integration in the Marcus sense. Marcus integrals involve sums over infinitely many jumps and transform under the usual laws of calculus \cite{applebaum2009levy}, and thus play a similar role for L\'evy processes as the Stratonovich integral for Brownian motion. The solution map of Marcus differential equations is continuous with respect to the driving process under certain variants of the Skorokhod topology. 
Note that if $b(X_t)$ is constant in $X_t$, then there is no difference in the solutions of the SDE in the sense of Marcus, It\^o and Stratonovich. We refer to Appendix \ref{app_sec:technical} for more details on these.

\paragraph{Generating stable laws using the Thaler map.} Let $\gamma := 1/\alpha$. It was shown in \cite{gottwald2021simulation} that $\alpha$-stable laws with $\alpha \in (1,2)$ can be generated using a deterministic dynamical system whose states $y_k$ are obtained by iterating the Thaler map $T: [0,1] \to [0,1]$ \cite{thaler1980estimates}: 
\begin{equation} \label{eq_Thaler}
T(y) = (y^{1-\gamma} + (1+y)^{1-\gamma} - 1)^{1/(1-\gamma)} \mod 1 \, , \, y_{k+1} = T(y_k),
\end{equation}
with any source of randomness coming solely from the initial condition $y_0 \in [0,1]$. The map has the following properties. Let $y^* \in (0,1)$ be the unique solution to the equation
$$(y^*)^{1-\gamma} + (1+y^*)^{1-\gamma} = 2.$$
There are two increasing branches  on the intervals $[0,y^*]$, $[y^*, 1]$.
For $\gamma \in [0,1)$, there exists a unique invariant probability density $h(y) = \frac{1-\gamma}{2^{1-\gamma}} (y^{-\gamma} + (1+y)^{-\gamma})$. Moreover, the density  defines a finite measure. For $\gamma \in (0,1)$, the map is nonuniformly expanding with a neutral fixed point at $y=0$  and correlations decay algebraically with rate $1/k^{\alpha - 1}$ (which is sharp) \cite{gottwald2021simulation}. Heuristically, the trajectory spends prolonged period of times near $y=0$ (laminar dynamics near the fixed point), slowing down the decay of correlation as $\alpha$ decreases.

Let the observable $v: [0,1] \to \mathbb{R}$ be a H\"older function with zero mean with respect to the invariant measure of $T$, i.e., $\ds\nu = h \ds y$, and consider the sequence of random variables $(v\circ T^k(y_0))_{k\ge 0}$, with the randomness coming from the initial state $ y_0$. When $\gamma \in (1/2, 1)$, the correlations are not summable and CLT breaks down if $v(0) \neq 0$ (it ''sees'' the  fixed point at $y=0$). Heuristically, the Birkhoff sum $\sum_{j=0}^{k-1} v \circ T^j$ is ballistic  with almost linear behavior near $x=0$ and the small jumps of size $v(0)$ accumulate into a single large jump, contradicting CLT \cite{gottwald2021simulation}. Instead, if $v$ is properly normalized, then it was shown in \cite{gouezel2004central} that the one-sided stable limit law 
$\frac{1}{k^{\gamma}} \sum_{j=0}^{k-1} v \circ T^j \overset{\mathrm{(d)}}{\to} X_{\alpha, \beta}\, \text{ as } \, k \to \infty, \text{ with } \beta=\mathrm{sign}(v(0))$,  holds instead, where $X_{\alpha, \beta}$ denotes the stable law whose characteristic function is given by: 
\begin{equation} \label{eq:charfunc}
\mathbb{E}[e^{it X_{\alpha, \beta}}] = \exp(- |t|^\alpha (1-i \beta \, \mathrm{sign}(t) \tan(\alpha \pi/2))).
\end{equation}

In this paper, we focus on MPGD with chaotic components whose limiting behavior at each time is described by the class of stable laws $X_{\alpha, \beta}$ whose characteristic function is given by \eqref{eq:charfunc} with the stability parameter $\alpha \in (1, 2)$ (i.e., $\gamma \in (1/2, 1)$) and the skewness parameter $\beta \in [-1,1]$. The $X_{\alpha, \beta}$ is centered (i.e., $\mathbb{E} X_{\alpha, \beta} = 0$), totally skewed (or one-sided) if $\beta = \pm 1$, and symmetric if $\beta = 0$.

\section{Multiscale Perturbed Gradient Descent (MPGD)} \label{sec:MPGD}

In full generality, we start by considering the following family (parametrized by $m > 0$) of deterministic fast-slow dynamical systems on $\R^d \times Y$, with $Y$ a bounded metric space:
\begin{equation}\label{DS}
    \left\{ 
    \begin{aligned}
    \xn{k+1} &= \xn{k} + \frac1{m}a_m(\xn{k}) + \sum_{i=1}^q \frac{1}{m^{1/\alpha_i}}b_m^{(i)}(\xn{k})v_{i}(y_k^{(i)}),\\
    y_{k+1}^{(i)} &= T_i(y_k^{(i)}), \ \ i=1,2,\ldots,q,
    \end{aligned}
    \right.
\end{equation}
for $k=0,1,2,\dots$, where the $\alpha_i \in (1,2)$, $a_m: \R^d \to \R^d$, $b_m^{(i)}: \R^d \to \R^{d \times r_i}$, $v_{i}: Y \to \R^{r_i}$, and $T_i: Y \to Y$. The initial conditions $x_0^{(m)}, y_0^{(i)}$ are independent random variables.  One important example of recursions of the form \eqref{DS} is the GD defined in \eqref{eqn:sgd}.
Other than GD, the above family of recursions \eqref{DS} includes various variants of GD  as special cases upon specifying suitable choices of the $a_m$, $b_m^{(i)}$, $v_i$, $T_i$, and the dimensions. 

Of particular interest are new algorithms that can be studied within the above framework. To this end, we introduce {\it multiscale perturbed gradient descent} (MPGD), whose dynamics are described by the following recursion:
\begin{equation} \label{eq_ouralg}
    \left\{ 
    \begin{aligned}
    \xn{k+1} &= \xn{k} - \frac{1}{m}  \nabla \hat{\mathcal{R}}(\xn{k}, S_n)  -  \frac{\mu}{m^{\frac{1}{\alpha_1}}} v_{1}(y_k^{(1)}) \odot \xn{k}   + \frac{\sigma}{m^{\frac{1}{\alpha_2}}} v_{2}(y_k^{(2)}),\\
    y^{(1)}_{k+1} &= T(y_k^{(1)}), \ \  y^{(2)}_{k+1} = T(y_k^{(2)}), \ \ k=0,1,2,\dots,
    \end{aligned}
    \right.
\end{equation}
where $T$ is the Thaler map \eqref{eq_Thaler} with $\gamma = \alpha^{-1} \in (1/2, 1)$, the $v_1, v_2$ are observable maps, $\mu \geq 0$ and  $\sigma \in \mathbb{R}$ are tunable parameters.
The above recursion is of the form \eqref{DS}, with $q=2$, $r_1 = r_2 = d$, $a_m(x) = -  \nabla \hat{\mathcal{R}}(x, S_n)$, $b_m^{(1)}(x) = -\mu \  \rm{diag}(x)$, and $b_m^{(2)}(x) = \sigma I$ independent of $m$. Another class of algorithms that falls within this framework is when $m$ is set to $n$, the number of training data points, in which case the empirical risk converges to the population risk as $m=n \to \infty$. 

To complete the description of MPGD, it remains to specify the choice for the observables $v_{i}(y_k^{(i)})$ in \eqref{eq_ouralg} whose rescaled version can be shown to generate stable laws $X_{\alpha,\beta}$, with $\alpha \in (1,2)$ and $\beta \in [-1,1]$, in the  limit $m \to \infty$. Following \cite{gottwald2021simulation}, we take them to be independent realizations of the observable $v^{(k)}$  described as follows. Let 
$d_\alpha = \frac{\alpha^\alpha (1-\gamma) \Gamma(1-\alpha) \cos(\alpha \pi/2)}{2^{1-\gamma}-1},$
and $\delta_0, \delta_1, \dots$ be independent copies of the random variable $\delta$ where $\mathbb{P}(\delta = \pm 1) = \frac{1}{2}(1 \pm \beta)$.
The observables $v^{(k)}$ are then defined to be: $v^{(k)} = \chi^{(k)} v \circ T^k$, where $v: [0,1] \to \mathbb{R}$ is the mean zero observable given by $v(y) =  d_\alpha^{-\gamma} (1-2^{\gamma-1})^{-\gamma} \tilde{v}(y)$,  with
$$\tilde{v}(y) = \begin{cases}
  1 & \text{if $y \leq y^*$}, \\
  (1-2^{1-\gamma})^{-1} & \text{if $y > y^*$},
\end{cases} \hspace{0.3cm}
\chi^{(k)} = \chi_{k-1} \cdots \chi_0 \in \{\pm 1 \} 
\hspace{0.1 cm} \text{ with } \ \chi_j = \begin{cases}
  1 & \text{if $T^j y \leq y^*$}, \\
  \delta_j & \text{if $T^j y > y^*$}.
\end{cases} $$ 
In particular, the random variables $\chi^{(k)}$ get updated only when the trajectory visits $(y^{*}, 1]$ and are not changed during the phase in $[0, y^*]$. 

The initial condition can be, in theory, equally well chosen using the invariant probability measure $\nu$ or the uniform Lebesgue measure. Empirically, convergence of the probability density is faster if it is drawn using $\nu$. Hence, we consider initial conditions drawn using $\nu$. Therefore, for MPGD we propagate uniformly distributed initial conditions $y_0' \in [0,1]$ under 10,000 iterations of the Thaler map and take  the initial condition as $y_0 = T^{10000} y_0'$.
Figure \ref{fig_dynamics} illustrates the dynamics of a realization of the Thaler iterates and the observable (perturbations used in MPGD), as well as the asymmetry and heavy tail nature of the distribution of the corresponding Birkhoff sum.

\begin{figure}[!t] 
\centering
\includegraphics[width=0.33\textwidth]{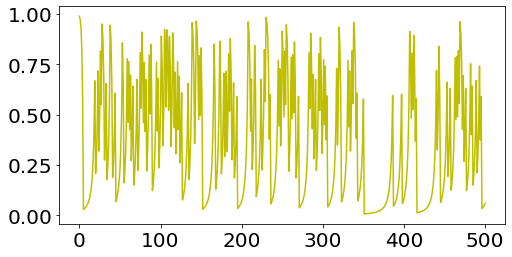}
\includegraphics[width=0.32\textwidth]{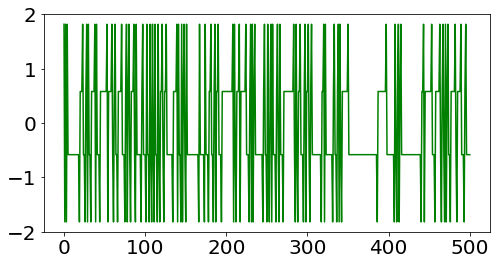}
\includegraphics[width=0.335\textwidth]{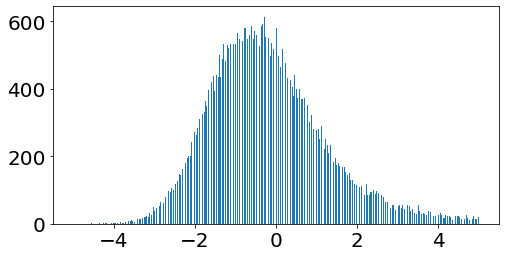}
\caption{A realization of the Thaler iterates (left) and the observable $v$ (middle) for $\gamma = 0.6$, $\beta = 0.1$. Empirical distribution of the corresponding Birkhoff sum $\frac{1}{k^{\gamma}} \sum_{j=0}^{k-1} v^{(j)}$ with $k=1000$ (right).}
\label{fig_dynamics}
\end{figure}

\section{Theoretical Analysis}

In this section, we analyze  MPGD via three different angles. We start by presenting results on superdiffusive limit (homogenization \cite{pavliotis2008multiscale}) for a class of fast-slow deterministic systems that include MPGD as a special case in Subsection \ref{subsec:limit}. We then study generalization properties of MPGD by  deriving a generalization bound for the limiting dynamics in Subsection \ref{subsec:generalization}. Lastly, we study the regularizing effects of MPGD via the lens of implicit regularization in Subsection \ref{subsec:implreg}.

\subsection{Superdiffusive Limits for the Fast-Slow Systems and MPGD} \label{subsec:limit}

Without loss of generality, we set $q=1$ in \eqref{DS} and consider the following family of  recursions:
\begin{equation}\label{DS_1var}
    \left\{ 
    \begin{aligned}
    \xn{k+1} &= \xn{k} + m^{-1}a_m(\xn{k}) + m^{-\frac{1}{\alpha}}b_m(\xn{k})v(y_k),\\
    y_{k+1} &= T(y_k),
    \end{aligned}
    \right.
\end{equation}
where $T$ is the Thaler map and $v$ is the observable constructed in Section \ref{sec:MPGD}.



The following result shows that stochastic dynamics emerges in the above family of deterministic dynamical systems in the limit $m \to \infty$. Moreover, the limiting dynamics are described by a stochastic differential equation (SDE) driven by a heavy-tailed L\'evy-stable process.

\begin{theorem}[Superdiffusive Limit -- Informal] \label{thm:main}
Let $\alpha\in (1,2)$ 
and assume that the coefficients  $a_m$, $b_m$ and the initial condition $x_0^{(m)}$ of \eqref{DS_1var} are well-behaved as $m \to \infty$. Then, under certain regularity conditions on the $a_m$, $b_m$, the process $(x_{\lfloor mt \rfloor}^{(m)})_{t\geq 0}$ converges in distribution  to the solution $(X_t)_{t \geq 0}$ of the SDE:
\[
\ds X_t= a(X_t) \ds t + b(X_t)\diamond \ds \mathrm{L}_t^{\alpha,\beta} \text{ with } X_0=x_0 \in \R^d,
\]
as $m \to \infty$, where $a = \limm a_m$, $b = \limm b_m$, $x_0 = \lim\limits_{m \to \infty} x_0^{(m)}$, $\diamond$ denotes the Marcus integration, defined in Section \ref{sec:Marcus}, and $(\mathrm{L}^{\alpha,\beta}_t)_{t\ge 0}$ is a L\'evy process with the characteristic exponent  \begin{equation}\label{eq:levy_exp}
\Psi(\xi) = |\xi|^\alpha (1-i \beta \, \mathrm{sign}(\xi) \tan(\alpha \pi/2)).
\end{equation}
\end{theorem}









In particular, 
Theorem \ref{thm:main} implies convergence in distribution of $x^{(m)}_{\lfloor mt \rfloor}$ to $X_t$ at fixed $t$. 
It is a special case of Theorem \ref{thm:main_app}, which covers a large class of the maps $T$ and observables $v$, and extends the results in \cite{gottwald2021simulation, chevyrev2020superdiffusive} to the case where the coefficients are dependent on $m$. We provide technical details on the topology used for the convergence and the proof in Appendix \ref{app_sec:technical}-\ref{sec:homogenization}. 

Specializing Theorem \ref{thm:main} (extended to the case of $q=2$) to the MPGD recursion \eqref{eq_ouralg}, we see that the rescaled MPGD iterate process, $(x_{\lfloor mt \rfloor}^{(m)})_{t\geq 0}$, converges in distribution to the solution $(X_t)_{t \geq 0}$ of the following SDE: 
$$\ds X_t = -  \nabla \hat{\mathcal{R}}(X_t, S_n) \ds t - \mu \ {\rm{diag}}(X_t) \diamond \ds \mathrm{L}_t^{\alpha_1, \beta_1} +  \sigma \ds \tilde{\mathrm{L}}_t^{\alpha_2, \beta_2}, \ \ X_0 = \lim_{m \to \infty} x_0^{(m)},$$ 
in the limit $m \to \infty$, which is equivalent to sending the step-size, $\eta := 1/m$, in the MPGD to zero. Here  $(\mathrm{L}_t^{\alpha_1, \beta_1})_{t \geq 0}$ and $(\tilde{\mathrm{L}}_t^{\alpha_2, \beta_2})_{t \geq 0}$ are independent L\'evy processes with characteristic exponent \eqref{eq:levy_exp}, and $\beta_1$, $\beta_2 \in [-1,1]$ are the parameters appearing in the construction of the observables $v_1$ and $v_2$ respectively. Therefore, we see that  MPGD converges to a gradient flow driven by  heavy-tailed L\'evy-stable processes in the considered scaling limit. The driving processes contain both additive and multiplicative components. The multiplicative component can be interpreted as a dynamical version of the weight decay, whereas the additive component can be viewed as  gradient noise \cite{neelakantan2015adding}.










\subsection{Generalization Properties of MPGD} \label{subsec:generalization}

In this section, we shift our focus to the generalization properties of MPGD. Our main roadmap will be to consider the limiting SDE arising from Theorem~\ref{thm:main} (with $\beta = 0$) and relate it to the recently developed generalization bounds for heavy-tailed random processes \cite{csimcsekli2020hausdorff,hodgkinson2021generalization}.

For mathematical convenience, we consider the following special case of \eqref{DS}\footnote{We believe a similar result would hold for \eqref{DS_1var} without symmetrization, which we leave for future work.}:
\begin{equation}\label{DS_sym}
    \left\{ 
    \begin{aligned}
    \xn{k+1} &= \xn{k} + m^{-1}a_m(\xn{k}) + m^{-\frac{1}{\alpha}}b_m(\xn{k})\left(v(y^{(1)}_k)- v(y^{(2)}_k) \right) ,\\
    y^{(1)}_{k+1} &= T(y^{(1)}_k), \quad
    y^{(2)}_{k+1} = T(y^{(2)}_k)
    \end{aligned}
    \right.
\end{equation}
In the light of Theorem~\ref{thm:main}, the recursion~\eqref{DS_sym} converges weakly to the following SDE: 
\begin{equation}\label{eq:sde}
\ds X_t= a(X_t)\ds t + b(X_t)\diamond \ds \mathrm{L}_t^{\alpha}, 
\end{equation}
where $\mathrm{L}^\alpha_t$ denotes the \emph{symmetric} $\alpha$-stable process in $\R^d$ (i.e., $\beta =0$).

Following the previous work \cite{csimcsekli2020hausdorff,hodgkinson2021generalization}, we are interested in bounding the worst-case generalization error over the trajectories of \eqref{eq:sde}. More precisely, let $(X_t)_{0\leq t \leq 1}$ denote the solution of \eqref{eq:sde}\footnote{Here, the range $[0,1]$ is arbitrary and could be replaced with any range $[0,T]$ for $0<T<\infty$.}, and define the set $\mathcal{X}$ so that it contains all the points visited by the trajectory:
\begin{align}
\mathcal{X} := \{x \in \mathbb{R}^d: \exists t \in [0,1], X_t = x\}.    
\end{align}
Then, our goal is to analyze $\sup_{x\in \mathcal{X}}|\hat{\mathcal{R}}(x,S_n) - \mathcal{R}(x)|$. Since we are mainly interested in the case when $a(x) = -\nabla \hat{\mathcal{R}}(x,S_n)$, the SDE \eqref{eq:sde} and hence the trajectory $\mathcal{X}$ will depend on the dataset $S_n$. Therefore, in the generalization bound, we will need to control the statistical dependence of $\mathcal{X}$ on $S_n$, through the notion of $\rho$-mutual information, defined as follows. Let $X,Y$ be two random elements, let $\mathbb{P}_{X,Y}$ denote their joint distribution, and let $\mathbb{P}_{X},\mathbb{P}_{Y}$ denote their respective marginal distributions. Then the $\rho$-mutual information between $X$ and $Y$ is defined as: $I_{\rho}(X,Y)=D_{\rho}(\mathbb{P}_{X,Y}\Vert\mathbb{P}_{X}\otimes\mathbb{P}_{Y})$, where $D_\rho$ is the $\rho$-Renyi divergence:\footnote{As $\rho \to 1$, $D_{\rho}$ tends to the Kullback-Leibler divergence.}
    $D_{\rho}(\mu,\nu)=\frac{1}{\rho-1}\log\mathbb{E}\bigg[\frac{\mathrm{d} \mu}{\mathrm{d} \lambda}(Z)^\rho \frac{\mathrm{d} \nu}{\mathrm{d} \lambda}(Z)^{1-\rho}\bigg]$,
and $\frac{\mathrm{d} \mu}{\mathrm{d} \lambda}$ denotes the Radon-Nikodym derivative. 
%
%
%
%

Finally, we require two technical regularity assumptions (Assumptions \ref{assmp:traj_regularity} and \ref{assmp:haus_dim}, given in Appendix \ref{sec:proofgen}) for our generalization bound. Informally, Assumption \ref{assmp:traj_regularity} is a topological regularity condition over the trajectory $\mathcal{X}$, whereas Assumption~\ref{assmp:haus_dim} is a statistical regularity assumption and imposes that in the neighborhood of a local minimum $x^\star$ of $\hat{\mathcal{R}}(x,S_n)$, the statistical behavior of the process $X_t$ can be approximated by the process solving the SDE
    $\mathrm{d} \tilde{X}_t = b(x^\star) \diamond \ds \mathrm{L}_t^{\alpha}$. This requires $b$ to be regular in the neighborhood of $x^\star$ and further imposes that $a(x^\star) = 0$, which is natural when $a(x) = -\nabla \hat{\mathcal{R}}(x,S_n)$. 
%
We now ready to present our generalization bound. 

\begin{theorem}[Generalization bound]
\label{thm:gen}
Assume that Assumptions \ref{assmp:traj_regularity} and \ref{assmp:haus_dim} given in Section~\ref{sec:proofgen} hold, and that $\ell$ is bounded by $B > 0$ and $L$-Lipschitz continuous. There exists a constant $K_1 > 0$ such that with probability at least $1 - \delta$, 
\begin{equation}
\label{eq:GenBoundProb}
\sup\nolimits_{x \in \mathcal{X}} |\hat{\mathcal{R}}(x, S_n) - \mathcal{R}(x)|  \leq K_1 \max\{B,L\} \left(\sqrt{\frac{   \alpha}{n}} + \sqrt{\frac{\log(1/\delta) + I_\infty(S_n, \mathcal{X})}{n}}\right),
\end{equation}
where $I_\infty(X,Y) := \lim_{\rho \to \infty} I_\rho(X,Y)$. Furthermore, there exists 
$K_2 > 0$ such that
\begin{equation}
\label{eq:GenBoundExp}
\mathbb{E}\sup\nolimits_{x \in \mathcal{X}} |\hat{\mathcal{R}}(x, S_n) - \mathcal{R}(x)| \leq K_2 \max\{B,L\} \left(\sqrt{\frac{\alpha }{n}} + \sqrt{\frac{I_1(S_n,\mathcal{X})}{n}}\right).
\end{equation}
\end{theorem}
This result shows that the generalization error is mainly bounded by two terms: the tail exponent of the limiting L\'{e}vy process $\mathrm{L}^\alpha_t$, and the statistical dependence of $\mathcal{X}$ on $S_n$. The bound suggests that, provided the mutual information between $\mathcal{X}$ and $S_n$ is fixed, the generalization error can be reduced by using a smaller $\alpha$ in MPGD. However, it has been illustrated that a smaller $\alpha$ can induce significant error on $\hat{\mathcal{R}}(x,S_n)$ \cite{simsekli2020fractional,camuto2021asymmetric}, hence a reasonable value of $\alpha$ should be chosen to balance both the empirical risk and the generalization error. The algorithm parameters $\mu, \sigma$ interact with the bound via the constant terms $K_1$, $K_2$  and the mutual information terms $I_1$, $I_\infty$. Unfortunately, their dependence is quite implicit and we are not able to provide quantitative estimations.

Finally, the main difference between Theorem~\ref{thm:gen} and \cite{csimcsekli2020hausdorff,hodgkinson2021generalization} is that, in our bounds we have explicit access to the characteristic exponent of $\mathrm{L}_t^\alpha$ since we construct it manually through our dynamics; whereas in the prior work generic processes were used as approximations for SGD, hence their results are not as explicit.

\subsection{Implicit Regularization of MPGD}
\label{subsec:implreg}
To understand MPGD and its connection to the vanilla GD \eqref{eqn:sgd} better, we provide additional analysis for the behavior of loss functions optimized under the recursion \eqref{eq_ouralg} via the lens of implicit regularization \cite{Mah12, lim2021noisy}. By this, we mean regularization imposed implicitly by the  learning strategy, without explicitly modifying the loss. We shall achieve this by deriving an appropriate explicit regularizer through a perturbation analysis in the weak perturbation regime. 

To this end, we work in the regime where the parameters $\mu, \sigma$ are small, keeping $m$ fixed. We set $\mu = \mu_0 \epsilon$ and $\sigma = \sigma_0 \epsilon$, where $\epsilon > 0$ is a small parameter, and analyze the loss optimized using the recursion \eqref{eq_ouralg} in the small $\epsilon$ regime. In the sequel, we let  $\overline{x}_k^{(m)}$ denote the states of the unperturbed GD, satisfying the recursion $\overline{x}_{k+1}^{(m)} = \overline{x}_k^{(m)} - \frac{1}{m} \hat{\mathcal{R}}(\overline{x}_k^{(m)}, S_n)$, for $k = 0,1, \dots$.
To simplify notation, we shall denote $  \hat{\mathcal{R}}(x) := \hat{\mathcal{R}}(x, S_n)$.

The following result relates the loss function, averaged over realizations of the
injected perturbations, evolved under training with MPGD to that of unperturbed GD in the weak perturbation regime.

\begin{theorem}[Implicit regularization] \label{thm_implreg}
Let $m > 0$, $k \in \mathbb{N}$, and $S_n$ be given. Assume that $r_1 = 1$ and $r_2 = d$ in Eq. \eqref{eq_ouralg}, and $\mathbb{E}[v_1(y^{(1)}_k)] = \mathbb{E}[v_2(y^{(2)}_k)] = 0  $ for all $k$. Then, for a scalar-valued loss function $\hat{\mathcal{R}}$ with $\nabla \hat{\mathcal{R}}$ having Lipschitz continuous partial derivatives in each coordinate up to order three (inclusive),
\begin{align}
    \mathbb{E}[\hat{\mathcal{R}}(x_k^{(m)})] &=  \hat{\mathcal{R}}(\overline{x}_k^{(m)})  + \frac{\epsilon^2}{2} \left( tr\left( C_k^{(m)} \nabla^2 \hat{\mathcal{R}}(\overline{x}_k^{(m)}) \right)  -   \nabla \hat{\mathcal{R}}(\overline{x}_k^{(m)})^T  \lambda_k^{(m)} \right) +  \mathcal{O}(\epsilon^3), \label{thm3_eqn}
\end{align}
as $\epsilon \to 0$, where the expectation is with respect to the realization of the $y^{(i)}$, the $\lambda_k^{(m)}$ are the vector with the $l$th component:
$$ [\lambda_k^{(m)}]^l = \frac{1}{m} \sum_{i=1}^k \sum_{j = 1}^d  [\Phi_{i}^{(m)}]^{lj} tr\left( C_{i-1}^{(m)} \nabla^2 [\nabla \hat{\mathcal{R}}(\overline{x}_{i-1}^{(m)}) ]^j  \right), $$
the $\Phi_{i}^{(m)} := \prod_{j=i}^{k-1} (I-\frac{1}{m} \nabla^2  \hat{\mathcal{R}}(\overline{x}_{j}^{(m)}))$, with the empty product taken to be the identity,
and the $C_k^{(m)}$ are covariance matrices with the $(p,q)$-entry of $\sum_{i_1,i_2 =1}^k \sum_{r,s=1}^d [\Phi_{i_1}^{(m)}]^{pr} [\Phi_{i_2}^{(m)}]^{qs} \theta_{i_1, i_2, r, s}^{(m)}$, where
\begin{align*}
    \theta_{i_1, i_2, r, s}^{(m)} &= \frac{\mu_0^2}{m^{2/\alpha_1}} \mathbb{E}[v_1(y_{i_1 - 1}^{(1)}) v_1(y_{i_2 - 1}^{(1)})] \cdot [x_{i_1 - 1}^{(m)}]^r [x_{i_2 - 1}^{(m)}]^s + \frac{\sigma_0^2}{m^{2/\alpha_2}} \mathbb{E}[[v_2]^r(y_{i_1 - 1}^{(2)}) [v_2]^s(y_{i_2 - 1}^{(2)})]  . 
\end{align*} 
\end{theorem}

Theorem \ref{thm_implreg} implies that the loss function optimized under MPGD at a given iteration and $m$  is, on average, approximately equivalent to a regularized objective functional. Moreover,  $\frac{\nabla \hat{\mathcal{R}}(\overline{x}_k^{(m)})^T  v_k^{(m)}}{tr\left( C_k^{(m)} \nabla^2 \hat{\mathcal{R}}(\overline{x}_k^{(m)}) \right)} \leq \frac{C}{m}$ for some constant $C$ independent of $m$, suggesting that the trace term is the dominant explicit regularizer for large $m$. This explicit regularizer is solely determined by the discrete-time flow generated by the Jacobians $I-\frac{1}{m} \nabla^2  \hat{\mathcal{R}}(\overline{x}_{j}^{(m)})$, the covariances $C_k^{(m)}$, and the Hessian of the loss function, all evaluated along the dynamics of the unperturbed GD. We can therefore expect the use of the perturbations in MPGD as a regularization mechanism should reduce the  Hessian of the loss function according to the perturbation levels ($\mu_0, \sigma_0$) and the correlations $\frac{1}{m^{2/\alpha_1}} \mathbb{E}[v_1(y_{i_1 - 1}^{(1)}) v_1(y_{i_2 - 1}^{(1)})]$, $\frac{1}{m^{2/\alpha_2}} \mathbb{E}[v_2(y_{i_1 - 1}^{(2)}) v_2(y_{i_2 - 1}^{(2)})]$. Reducing the Hessian of the loss function can lead to flatter minima in the loss landscape, which is widely believed to associate with better generalization in deep learning \cite{keskar2016large, jiang2019fantastic}. 

Lastly, we observe that in the weak perturbation regime and for non-convex losses, MPGD is more beneficial than GD in the sense that it reduces the Hessian of the loss, thereby promoting flatter minima. Hence, in order to observe the benefits of our scheme, we suspect that the optimization problem should be non-convex and local minima should have different curvatures.



\section{Empirical Results} \label{sect_empiricalresults}

In this section, we  illustrate the advantages of MPGD  compared to other schemes on the tasks of (1) minimizing the widening valley loss, (2) regression on the Airfoil Self-Noise Dataset, and (3) classification on CIFAR-10. We also provide additional results and details in Appendix \ref{app_sec:addexps}.

\subsection{Minimizing the Widening Valley Loss} 
We consider the problem of minimizing the widening valley loss, given by $\ell(u,v) = v^2 \|u\|^2/2$.  The gradient of $l$ is given by $\nabla \ell(u,v) = (v^2 u, \|u\|^2 v)$, and the trace of the Hessian is $d v^2 + \|u\|^2$, where $d$ is the dimension of $u$. The trace of the Hessian measures the flatness of the minima, which is monotonously changing in this case.   The loss has a valley of minima with zero loss for all $(u,v)$ with $v=0$. The smaller the norm of $u$, the flatter the minimum. Vanilla GD gets stuck when it first enters this valley, and the intuition is that injecting suitable perturbations should help convergence to a flat (with small $\|u\|$) part of the valley. All $(u,v)$ with $v=0$ are minima, but we also need $\|u\|$ to be minimized in order to minimize the trace of the Hessian. In high dimension (large $d$), the GD path is biased towards making $v$ small and not optimizing $u$ since the direction along $v$ is the most curved.

For the experiments, we start optimizing from the point $(u_0,0)$, where $u_0 \sim 5 \cdot \mathcal{U}(d)$ with $d=10$, and use the learning rate $\eta = 0.01$.  We study and compare the behavior of the following  schemes: (i) baseline (vanilla GD), (ii) GD with uncorrelated Gaussian noise injection instead, and (iii) MPGD. Figure~\ref{fig_wideningvalleyloss} demonstrates that MPGD can lead to successful optimization of the widening loss whereas the baseline GD and GD with Gaussian perturbations lead to poor solutions. This is in agreement with our analysis of implicit regularization for MPGD, showing that the injected perturbations effectively favor small trace of the loss Hessian, thereby biasing the solution to flatter region of the loss landscape.

\clearpage 

\begin{figure}[!t] 
\hspace{-0.3cm}
\includegraphics[width=0.3\textwidth]{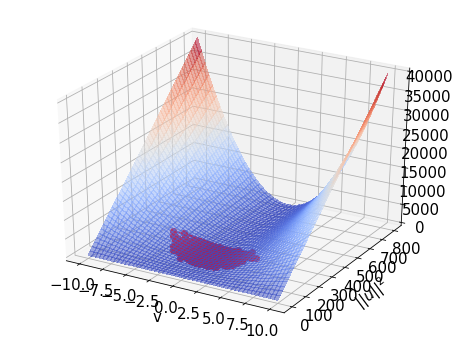}
\hspace{-0.2cm} \includegraphics[width=0.36\textwidth]{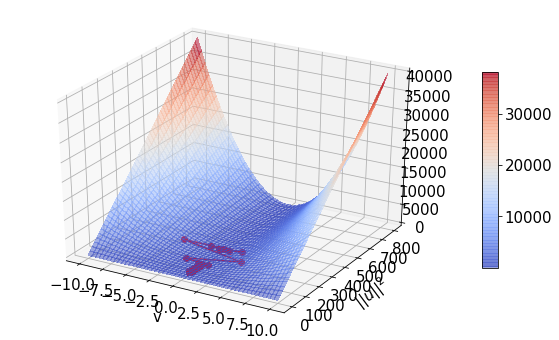}
\includegraphics[width=0.34\textwidth]{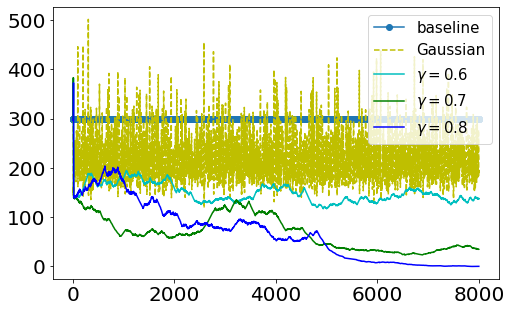}
\caption{Evolution of GD with Gaussian perturbations (left plot) vs. that of MPGD with $\gamma = 0.7$, $\beta = 0.5$ (middle plot), using $\mu = 0.02$, $\sigma = 0.05$, and $\eta = 0.01$ ($m = 100$). Here we see that MPGD leads to successful optimization of the widening valley loss whereas that with Gaussian perturbations fails to converge. Moreover, MPGD effectively reduces the trace of loss Hessian (see right plot), steering the GD iterates to flatter region of the loss landscape. }
\label{fig_wideningvalleyloss}
\end{figure}

\subsection{Airfoil Self-Noise Prediction}

\begin{wraptable}{r}{0.5\textwidth} 
    \vspace{-1.1cm}
	\caption{Shallow neural nets trained on the airfoil data set. The results in parenthesis are achieved with the variant \eqref{DS_sym}. 
	All the results are averaged over 5 models trained with different seed values. \vspace{0.3cm}}
	\label{tab:airfoil}
    \centering
    \scalebox{0.8}
    {
\begin{tabular}{@{}lcccc@{}}
\toprule
\multicolumn{1}{c}{Scheme} & test RMSE & RMSE gap \\ \midrule
Baseline     &  0.4309   &  0.2411   \\
Gaussian          &  0.4279 &  0.2354 \\
MPGD, $\gamma = 0.55$           &  0.3916 (0.3865)  &   0.2256  (0.2305)  \\
MPGD, $\gamma = 0.6$        &  {\bf  0.3810} (0.4092)   &  0.2298 ({\bf 0.2206}) \\
MPGD, $\gamma = 0.65$         &  0.3829 (0.3891)  &    0.2407 (0.2307) \\
MPGD, $\gamma = 0.7$    &   0.4600 (0.3754)  &  0.2315 (0.2311)    \\
\bottomrule 
\end{tabular}
}
\end{wraptable}

We consider the Airfoil Self-Noise Dataset \cite{Dua:2019} from the UCI repository. It comprises of differently sized airfoils at various wind tunnel speeds and angles of attack. This is a regression problem aiming to predict scaled sound pressure level, in decibels, of the airfoil based on the the  features: frequency, angle of attack, chord length,  free-stream velocity, and suction side displacement thickness. We use 1202 samples for training and 301 samples for testing. 

For the training, we use a fully connected shallow neural network of width 16 with  ReLU activation and train for 3000 epochs with the learning rate $\eta = 0.1$, using  mean square error (MSE) as the loss and choosing $\beta = 0.5$.  Table \ref{tab:airfoil} reports the average root MSE (RMSE) and the RMSE gap (defined as test RMSE - train RMSE) evaluated for models that are trained with 5 different seed values for this task. We can see that MPGD leads to both lower test RMSE and RMSE gap when compared with vanilla GD (baseline) and GD with uncorrelated Gaussian perturbations (see the results not in parenthesis in Table \ref{tab:airfoil};  here  $\mu = 0.01$, $\sigma = 0.02$). Using the form of the perturbations in \eqref{DS_sym} instead can also give lower  RMSE gap (see the results in parenthesis in Table \ref{tab:airfoil}; here  $\sigma = \mu = 0.01$). Overall these results support our generalization theory for MPGD. Lastly, we note that using larger values of $\gamma$ naively does not guarantee better test performance: one has to fine tune the parameters $\beta$, $\mu$, $\sigma$, $\eta$ appropriately to achieve favorable trade-off between training stability and test performance.

\subsection{CIFAR-10 Classification} 
We consider training a ResNet-18 classifier \cite{he2016deep} on the CIFAR-10 data set \cite{krizhevsky2009learning}. 
We follow the setup used in \cite{geiping2021stochastic}, where a ResNet-18 is trained using batch size of 50K (the entire training dataset). As in \cite{geiping2021stochastic}, we consider a standard ResNet-18 with the parameters in the linear layer randomly initialized and the parameters in the batch normalization initialized with mean zero and unit variance, except for the last in each residual branch which is initialized to zero. The default random CIFAR-10 data ordering is kept as is, and every image is randomly augmented by horizontal flips and random crops after padding by 4 pixels. 

Using the setup of \cite{geiping2021stochastic}, the reference mini-batch SGD  is trained using a batch size of 128 (sampling without replacement), Nesterov momentum of 0.9 and weight decay of 0.0005. The learning rate is warmed up from 0.0 to 0.1 over the first 5 epochs and then reduced via cosine annealing to 0 over the course of training for 300 epochs (resulting in 390x300=117,000 update steps). The  validation accuracy obtained in \cite{geiping2021stochastic} is 95.7\%. For the full-batch GDs, we replace the mini-batch updates by full batches and accumulate the gradients over all mini-batches.
For the Gaussian scheme and MPGD, we inject the perturbations into the GD recursion for  parameters of the input layer of ResNet-18, choosing $\mu = 0.03$, $\sigma = 0.01$ and $\beta = 0.5$ (for MPGD). In our experiments, we fix the learning rate to 0.1 throughout training, and do not add momentum and weight decay. This is in constrast to \cite{geiping2021stochastic}, where the additional tricks applied in the mini-batch SGD are also used for the full-batch GD, leading to higher test accuracies (see Table 2 in \cite{geiping2021stochastic}).

Table \ref{tab:cifar10-resnet18} shows that MPGD can lead to better test performance and lower accuracy gap when compared to the baseline (full batch GD) and Gaussian noise-perturbed GD. Using the perturbations in \eqref{DS_sym} instead can also achieve lower  accuracy gap. Therefore, adding perturbations of MPGD can be potentially a useful trick to improve training of deep architectures on benchmark data sets. Lastly, we note that additional tricks, such as gradient clipping and tuning of optimization hyperparameters, can be applied to improve our test accuracies and close the gap to results obtained with the reference mini-batch SGD. However, our main focus here is not on competitive performance, but rather on demonstrating the effects of the perturbations in MPGD when compared to full batch GDs  and Gaussian noise-perturbed full-batch GD to support our theory.

\begin{table}[!t]
	\caption{ResNet-18 trained on CIFAR-10 for 1000 epochs. Here, accuracy gap =  training accuracy - validation accuracy. The results in parenthesis are achieved with the variant of MPGD \eqref{DS_sym}. All the results are averaged over 5 models trained with different seed values.  }
	\label{tab:cifar10-resnet18}
\centering
\scalebox{0.8}{
\begin{tabular}{@{}lccc@{}}
\toprule
\multicolumn{1}{c}{Scheme}   & val. accuracy in \% & accuracy gap in \% \\ \midrule
Baseline full batch GD                    &  73.10 & 8.99   \\
Gaussian                         & 70.58 &              8.53   \\
MPGD, $\gamma = 0.55$      &  {\bf 75.97} (74.51) &  6.08 (6.58)               \\
MPGD, $\gamma = 0.6$   &  75.78 (74.60) &    {\bf  5.64} (6.15)        \\
MPGD, $\gamma = 0.65$                       &  74.39 (72.96) &   6.41 (6.85)           \\
MPGD, $\gamma = 0.7$  &  73.72 (73.27) &  6.69 (6.04)          \\ 
\bottomrule
\end{tabular}
}
\end{table}

Lastly, we remark that we did not tune the step-size for any of the algorithms. Since our generalization bound applies to the asymptotic SDE, which is obtained when the step-size goes to zero, in order to stay close to the theory, we chose a small enough step-size to be both not far from the continuous dynamics, and large-enough so that the algorithm converges in a reasonable amount of time. That being said, we have tried a range of step-sizes for both algorithms, and we observed that the proposed scheme consistently outperforms vanilla GD for smaller step-sizes. Whereas if we use a large step-size, both algorithms perform similarly. In this regime, vanilla GD as well potentially emits a chaotic behavior, which might also indicate the importance of the "implicit randomness". However, as we indicated in the introduction, this regime is not easily controllable, and our purpose is to introduce a controlled chaotic component, with a clear theoretical understanding.


\section{Conclusion}
In this paper, we introduce and study a class of slow-fast deterministic dynamical systems which homogenize to a limiting SDE driven by a heavy-tailed L\'evy-stable process in an appropriate scaling limit as a rigorous framework for perturbed GDs. Within this framework, we introduce MPGD, a novel version of perturbed GD, which we show to have good generalization  and regularization properties. We further demonstrate the advantages of MPGD  empirically in various optimization tasks. Our framework can provide useful tools for identifying implicit randomness in deterministic optimization algorithms and inspire other promising algorithms. It would be interesting to investigate the interactions of the chaotic perturbations in MPGD with the noise arising from mini-batch sampling in SGD \cite{wu2020noisy} and other optimization tricks. Since this is a theoretical paper studying a framework for perturbed GDs, there are no potential negative societal impacts of our work.

\paragraph{Acknowledgements.} 
We are grateful to the computational  resources provided by the Swedish National Infrastructure
for Computing (SNIC) at  Chalmers Centre for Computational Science and Engineering (C3SE)  partially funded by the Swedish Research Council through grant agreement no. 2018-05973. S. H. Lim would like to acknowledge the WINQ Fellowship, the SNIC AI/ML grant, and the Swedish Research Council (VR/2021-03648) for providing support of this work. U. \c{S}.'s research was funded in part by the French
government under management of Agence Nationale de la
Recherche as part of the ``Investissements d’avenir'' program, reference ANR-19-P3IA-0001 (PRAIRIE 3IA Institute), and the European Research Council Starting Grant DYNASTY -- 101039676.

\clearpage
\bibliographystyle{alpha}
{
\small
\bibliography{ref.bib}

\newcommand{\etalchar}[1]{$^{#1}$}
\begin{thebibliography}{WHX{\etalchar{+}}20}

\bibitem[AGZ21]{agarwal2021acceleration}
Naman Agarwal, Surbhi Goel, and Cyril Zhang.
\newblock Acceleration via fractal learning rate schedules.
\newblock {\em arXiv preprint arXiv:2103.01338}, 2021.

\bibitem[App09]{applebaum2009levy}
David Applebaum.
\newblock {\em L{\'e}vy Processes and Stochastic Calculus}.
\newblock Cambridge University Press, 2009.

\bibitem[BCN18]{bottou2018optimization}
L{\'e}on Bottou, Frank~E Curtis, and Jorge Nocedal.
\newblock Optimization methods for large-scale machine learning.
\newblock {\em SIAM Review}, 60(2):223--311, 2018.

\bibitem[C{\etalchar{+}}47]{cauchy1847methode}
Augustin Cauchy et~al.
\newblock M{\'e}thode g{\'e}n{\'e}rale pour la r{\'e}solution des systemes
  d’{\'e}quations simultan{\'e}es.
\newblock {\em Comp. Rend. Sci. Paris}, 25(1847):536--538, 1847.

\bibitem[CFK{\etalchar{+}}16]{chevyrev2016multiscale}
Ilya Chevyrev, Peter~K Friz, Alexey Korepanov, Ian Melbourne, and Huilin Zhang.
\newblock Multiscale systems, homogenization, and rough paths.
\newblock In {\em International Conference in Honor of the 75th Birthday of SRS
  Varadhan}, pages 17--48. Springer, 2016.

\bibitem[CFKM20]{chevyrev2020superdiffusive}
Ilya Chevyrev, Peter~K Friz, Alexey Korepanov, and Ian Melbourne.
\newblock Superdiffusive limits for deterministic fast--slow dynamical systems.
\newblock {\em Probability Theory and Related Fields}, 178(3):735--770, 2020.

\bibitem[CKL{\etalchar{+}}21]{cohen2021gradient}
Jeremy~M Cohen, Simran Kaur, Yuanzhi Li, J~Zico Kolter, and Ameet Talwalkar.
\newblock Gradient descent on neural networks typically occurs at the edge of
  stability.
\newblock In {\em ICLR}, 2021.

\bibitem[CP14]{chechkin2014marcus}
Alexei Chechkin and Ilya Pavlyukevich.
\newblock Marcus versus {S}tratonovich for systems with jump noise.
\newblock {\em Journal of Physics A: Mathematical and Theoretical},
  47(34):342001, 2014.

\bibitem[CWZ{\etalchar{+}}21]{camuto2021asymmetric}
Alexander Camuto, Xiaoyu Wang, Lingjiong Zhu, Chris Holmes, Mert
  G{\"u}rb{\"u}zbalaban, and Umut {\c{S}}im{\c{s}}ekli.
\newblock Asymmetric heavy tails and implicit bias in {G}aussian noise
  injections.
\newblock {\em arXiv preprint arXiv:2102.07006}, 2021.

\bibitem[DG17]{Dua:2019}
Dheeru Dua and Casey Graff.
\newblock {UCI} machine learning repository, 2017.

\bibitem[FV10]{friz_victoir_2010}
Peter~K. Friz and Nicolas~B. Victoir.
\newblock {\em Multidimensional Stochastic Processes as Rough Paths: Theory and
  Applications}.
\newblock Cambridge Studies in Advanced Mathematics. Cambridge University
  Press, 2010.

\bibitem[GAG{\etalchar{+}}00]{goldberger2000physiobank}
Ary~L Goldberger, Luis~AN Amaral, Leon Glass, Jeffrey~M Hausdorff, Plamen~Ch
  Ivanov, Roger~G Mark, Joseph~E Mietus, George~B Moody, Chung-Kang Peng, and
  H~Eugene Stanley.
\newblock Physio{B}ank, {P}hysio{T}oolkit, and {P}hysio{N}et: components of a
  new research resource for complex physiologic signals.
\newblock {\em Circulation}, 101(23):e215--e220, 2000.

\bibitem[GGP{\etalchar{+}}21]{geiping2021stochastic}
Jonas Geiping, Micah Goldblum, Phillip~E Pope, Michael Moeller, and Tom
  Goldstein.
\newblock Stochastic training is not necessary for generalization.
\newblock {\em arXiv preprint arXiv:2109.14119}, 2021.

\bibitem[GK54]{gnedenko_kolmogorov}
Boris~Vladimirovich Gnedenko and Andrey~Nikolaevich Kolmogorov.
\newblock {\em Limit Distributions for Sums of Independent Random Variables}.
\newblock Cambridge, Addison-Wesley, 1954.

\bibitem[GM13]{gottwald2013homogenization}
Georg~A Gottwald and Ian Melbourne.
\newblock Homogenization for deterministic maps and multiplicative noise.
\newblock {\em Proceedings of the Royal Society A: Mathematical, Physical and
  Engineering Sciences}, 469(2156):20130201, 2013.

\bibitem[GM21]{gottwald2021simulation}
Georg~A Gottwald and Ian Melbourne.
\newblock Simulation of non-{L}ipschitz stochastic differential equations
  driven by $\alpha$-stable noise: A method based on deterministic
  homogenization.
\newblock {\em Multiscale Modeling \& Simulation}, 19(2):665--687, 2021.

\bibitem[Gou04]{gouezel2004central}
S{\'e}bastien Gou{\"e}zel.
\newblock Central limit theorem and stable laws for intermittent maps.
\newblock {\em Probability Theory and Related Fields}, 128(1):82--122, 2004.

\bibitem[H{\c{S}}KM21]{hodgkinson2021generalization}
Liam Hodgkinson, Umut {\c{S}}im{\c{s}}ekli, Rajiv Khanna, and Michael~W
  Mahoney.
\newblock Generalization properties of stochastic optimizers via trajectory
  analysis.
\newblock {\em arXiv preprint arXiv:2108.00781}, 2021.

\bibitem[HZRS16]{he2016deep}
Kaiming He, Xiangyu Zhang, Shaoqing Ren, and Jian Sun.
\newblock Deep residual learning for image recognition.
\newblock In {\em Proceedings of the IEEE Conference on Computer Vision and
  Pattern Recognition}, pages 770--778, 2016.

\bibitem[JNM{\etalchar{+}}19]{jiang2019fantastic}
Yiding Jiang, Behnam Neyshabur, Hossein Mobahi, Dilip Krishnan, and Samy
  Bengio.
\newblock Fantastic generalization measures and where to find them.
\newblock {\em arXiv preprint arXiv:1912.02178}, 2019.

\bibitem[KMN{\etalchar{+}}16]{keskar2016large}
Nitish~Shirish Keskar, Dheevatsa Mudigere, Jorge Nocedal, Mikhail Smelyanskiy,
  and Ping Tak~Peter Tang.
\newblock On large-batch training for deep learning: Generalization gap and
  sharp minima.
\newblock {\em arXiv preprint arXiv:1609.04836}, 2016.

\bibitem[Kri09]{krizhevsky2009learning}
Alex Krizhevsky.
\newblock Learning multiple layers of features from tiny images.
\newblock 2009.

\bibitem[KT20]{kong2020stochasticity}
Lingkai Kong and Molei Tao.
\newblock Stochasticity of deterministic gradient descent: Large learning rate
  for multiscale objective function.
\newblock {\em Advances in Neural Information Processing Systems},
  33:2625--2638, 2020.

\bibitem[LEHM21]{lim2021noisy}
Soon~Hoe Lim, N~Benjamin Erichson, Liam Hodgkinson, and Michael~W Mahoney.
\newblock Noisy recurrent neural networks.
\newblock {\em Advances in Neural Information Processing Systems}, 34, 2021.

\bibitem[Mah12]{Mah12}
M.~W. Mahoney.
\newblock Approximate computation and implicit regularization for very
  large-scale data analysis.
\newblock In {\em Proceedings of the 31st ACM Symposium on Principles of
  Database Systems}, pages 143--154, 2012.

\bibitem[Mar81]{marcus1981modeling}
Steven~I Marcus.
\newblock Modeling and approximation of stochastic differential equations
  driven by semimartingales.
\newblock {\em Stochastics: An International Journal of Probability and
  Stochastic Processes}, 4(3):223--245, 1981.

\bibitem[NVL{\etalchar{+}}15]{neelakantan2015adding}
Arvind Neelakantan, Luke Vilnis, Quoc~V Le, Ilya Sutskever, Lukasz Kaiser,
  Karol Kurach, and James Martens.
\newblock Adding gradient noise improves learning for very deep networks.
\newblock {\em arXiv preprint arXiv:1511.06807}, 2015.

\bibitem[Oym21]{oymak2021provable}
Samet Oymak.
\newblock Provable super-convergence with a large cyclical learning rate.
\newblock {\em IEEE Signal Processing Letters}, 28:1645--1649, 2021.

\bibitem[PS08]{pavliotis2008multiscale}
Grigoris Pavliotis and Andrew Stuart.
\newblock {\em Multiscale Methods: Averaging and Homogenization}.
\newblock Springer Science \& Business Media, 2008.

\bibitem[Sam17]{samoradnitsky2017stable}
G.~Samoradnitsky.
\newblock {\em Stable Non-Gaussian Random Processes: Stochastic Models with
  Infinite Variance}.
\newblock CRC Press, 2017.

\bibitem[Sch98]{schilling1998feller}
Ren{\'e}~L Schilling.
\newblock Feller processes generated by pseudo-differential operators: On the
  {H}ausdorff dimension of their sample paths.
\newblock {\em Journal of Theoretical Probability}, 11(2):303--330, 1998.

\bibitem[Smi17]{smith2017cyclical}
Leslie~N Smith.
\newblock Cyclical learning rates for training neural networks.
\newblock In {\em 2017 IEEE Winter Conference on Applications of Computer
  Vision (WACV)}, pages 464--472. IEEE, 2017.

\bibitem[{\c{S}}SDE20]{csimcsekli2020hausdorff}
Umut {\c{S}}im{\c{s}}ekli, Ozan Sener, George Deligiannidis, and Murat~A
  Erdogdu.
\newblock Hausdorff dimension, heavy tails, and generalization in neural
  networks.
\newblock {\em arXiv preprint arXiv:2006.09313}, 2020.

\bibitem[SSG19]{simsekli2019tail}
Umut Simsekli, Levent Sagun, and Mert Gurbuzbalaban.
\newblock A tail-index analysis of stochastic gradient noise in deep neural
  networks.
\newblock In {\em International Conference on Machine Learning}, pages
  5827--5837. PMLR, 2019.

\bibitem[SZTG20]{simsekli2020fractional}
Umut Simsekli, Lingjiong Zhu, Yee~Whye Teh, and Mert Gurbuzbalaban.
\newblock Fractional underdamped {L}angevin dynamics: Retargeting {SGD} with
  momentum under heavy-tailed gradient noise.
\newblock In {\em International Conference on Machine Learning}, pages
  8970--8980. PMLR, 2020.

\bibitem[Tha80]{thaler1980estimates}
Maximilian Thaler.
\newblock Estimates of the invariant densities of endomorphisms with
  indifferent fixed points.
\newblock {\em Israel Journal of Mathematics}, 37(4):303--314, 1980.

\bibitem[WHX{\etalchar{+}}20]{wu2020noisy}
Jingfeng Wu, Wenqing Hu, Haoyi Xiong, Jun Huan, Vladimir Braverman, and
  Zhanxing Zhu.
\newblock On the noisy gradient descent that generalizes as {S}{G}{D}.
\newblock In {\em International Conference on Machine Learning}, pages
  10367--10376. PMLR, 2020.

\bibitem[You36]{young1936}
L.~C. Young.
\newblock {An inequality of the Hölder type, connected with Stieltjes
  integration}.
\newblock {\em Acta Mathematica}, 67(none):251 -- 282, 1936.

\end{thebibliography}
}


\newpage
\appendix

\begin{center}
    {\Large \bf  Appendix}
\end{center}

\appendix

\section{Overview}
\label{app_sect_overview}

\noindent {\bf Organizational Details.} This Appendix is organized as follows.
\begin{itemize}
    \item In Section \ref{app_sec:technical}, we provide technical background materials that will be needed for proving the homogenization results.
    \item In Section \ref{sec:homogenization}, we present a  general theorem that includes Theorem \ref{thm:main} as a special case and provide the proof. 
    \item In Section \ref{sec:proofgen}, we provide proof for the generalization bound (Theorem \ref{thm:gen}).
    \item In Section \ref{app_sec:implreg}, we provide proof for the implicit regularization result (Theorem \ref{thm_implreg}).
    \item In Section \ref{app_sec:addexps}, we provide additional empirical results and details. 
\end{itemize}

In addition to the notations introduced in the main paper, we shall need the following notations.

\noindent {\bf Notation.}
Denote the space of c\`adl\`ag (right continuous with left limits) functions from $[0,1]$ to $\R^d$ by $D([0,1],\R^d)$, on which the Skorokhod-type metrics are denoted by $\sigma_\cdot$. The left limit of $f\in D([0,1],\R^d)$ at $t$ is written by $f(t-):= \lim_{s\uparrow t}f(s)$. The identity function is denoted by~$\mathrm{id}: t\mapsto t,\, t\in[0,1]$. For a continuous function $f$ defined on $\R^d$, denote the uniform norm by $\| f\|_\infty = \sup_{x\in \mathbb{R}^d}|f(x)|$. Denote by $\mathscr{D}([0,1],\R^d)$ the space of admissible path-trajectory pairs in Definition \ref{def:admissible_pair}, whose metrics are denoted by $\alpha_\cdot$. Denote by $C^{p\var}([0,1],\R^d)$ the space of continuous functions with finite $p$-variation. In Definition \ref{def:path_function}, denote by $l_d$ the linear interpolation in $\R^d$.

\section{Additional Technical Background}
\label{app_sec:technical}

\subsection{Background on metrics and topologies on the space of c\`adl\`ag functions}
\subsubsection{Skorokhod-type topologies}
 The Skorokhod $\mathcal{J}_1$ topology on $D([0,1],\R^d)$ is induced by the following metric. 
\begin{definition}[Skorokhod distance]
 The Skorokhod distance on $D([0,1],\R^d)$, the space of c\`adl\`ag functions, is defined by
 \[
\sigma_\infty(X_1,X_2) = \inf \|\lambda-\mathrm{id}\|_\infty \vee \|X_1\circ\lambda, X_2\|_\infty,
\]
where the inf is taken on all increasing bijections $\lambda$ from $[0,1]$ to itself.
\end{definition}

Another important topology on $D([0,1], \R^d)$ is the $\mathcal{SM}_1$ topology defined as follows.
\begin{definition}[Skorokhod $\mathcal{SM}_1$ topology]\label{def:sm1}
The Skorokhod $\mathcal{SM}_1$ topology on $D([0,1],\R^d)$, the space of c\`adl\`ag functions, is defined by the metric 
\[
d_{\mathcal{SM}_1}(X_1,X_2) = \inf\|(\lambda_1,\gamma_1)-(\lambda_2,\gamma_2)\|_\infty,
\]
where the inf is taken on all pairs $(\lambda_i,\gamma_i)\in C([0,1],[0,1]\times\R^d)$ such that $(\lambda_i,\gamma_i)(0)=(0,X_i(0))$, $(\lambda_i,\gamma_i)(1)=(1,X_i(1))$ and $\gamma_i(t)\in [X_i(\lambda_i(t)_-),X_i(\lambda_i(t))]$, $i=1,2$. It is equivalent to the metric on the graph of functions in $D$, with discontinuities connected by the straight segments.
\end{definition}

Recall that for $1\le p< \infty$, the \emph{$p$-variation} of $u:[0,1]\to \R^d$ is given by
\[
\|u\|_{p\var}= \sup_{0=t_0<t_1<\ldots<t_k=1} \left(\sum_{j=1}^k |u(t_j)-u(t_{j-1})|^p\right)^{1/p},
\]
and the subspace of $D([0,1],\R^d)$ with finite $p$-variation is denoted by $D^{p\var}([0,1],\R^d)$. Let us define the following $p$-variation generalisations of the Skorokhod topology.
\begin{definition}[Skorokhod-type $p$-variation]
\[
\sigma_{p\var}(X_1,X_2) = \inf\max\{\|\lambda-\mathrm{id}\|_\infty,\|X_1\circ\lambda-X_2\|_{p\var}\},
\]
where the inf is taken on all increasing bijections $\lambda$ from $[0,1]$ to itself.
\end{definition}

\subsubsection{Generalised Skorokhod topologies with interpolations}

For discontinuous c\`adl\`ag functions, one can interpolate jumps with path functions.
\begin{definition}[Path function]\label{def:path_function}
A \emph{path function} on $\R^d$ is a map $\phi: J \to C([0,1],\R^d)$, where $J\subset \R^d\times\R^d$, for which $\phi(x,y)(0)=x$ and $\phi(x,y)(1)=y$ for all $(x,y)\in J$.
\end{definition}

\begin{definition}\label{def:admissible_pair}
A pair $(X,\phi)$ is called admissible if all the jumps of $X$ are in the domain of definiton $J$ of $\phi$, i.e. $(X(t-),X(t))\in J$ for all jump times $t$ of $X$. Denote by $\mathscr{D}([0,1],\R^d)$ the space of admissible pairs, modulo the equivalence $(X_1,\phi_1)\sim (X_2,\phi_2)$ if $X_1=X_2$ and $\phi_1(X_1(t-),X_1(t))$ is a reparametrization of $\phi_2(X_1(t-),X_1(t))$ for all jump times $t$ of $X_1$.
\end{definition}

\begin{definition}\label{def:X^phi,delta}
For admissible pair $(X,\phi)\in\mathscr{D}([0,1],\R^d)$, we now construct a continuous path $X^{\phi,\delta}$ as follows.
\begin{itemize}
    \item Given a sequence $r_1,r_2,\ldots>0$ with $r:=\sum_j r_j<\infty$,
\[
\text{Let }\tau:[0,1]\to[0,1+r] \text{ given by } \tau(t)=t+\sum_k\delta r_k 1_{\{t_k\le t\}}.
\]
    \item Define an intermediate process $\hat{X}\in C([0,1+\delta r],\R^d)$,
    \[ \hat{X}(t)=
    \left\{ 
    \begin{aligned}
    &X(s)              &\text{ if } t=\tau(s) \text{ for some }s\in[0,1]\\
    &\phi(X(t_k-),X(t_k))\left(\frac{s-\tau(t_k-)}{\delta r_k}\right) &\text{ if }t\in[\tau(t_k-),\tau(t_k)) \text{ for some }k.
    \end{aligned}
    \right.
    \]
    \item Finally, let $X^{\phi,\delta}(t)=\hat{X}(t(1+\delta r))$. We will drop the superscript $\delta$ if $\delta=1$.
\end{itemize}
\end{definition}

\begin{definition}\label{def:p_var_metric}
The (pseudo)metric $\alpha_{p\var}$ on $\mathscr{D}([0,1],\R^d)$ is defined by
	\[
	\alpha_{p\var}((X,\phi), (\overline{X},\overline{\phi})):= \lim_{\delta\to 0}\sigma_{p\var}(X^{\phi,\delta}, \overline{X}^{\overline{\phi},\delta})
	\]
	independent of the choice of the series $\sum_{k=1}^\infty r_k$. Denote by $\mathscr{D}^{p\var}([0,1],\R^d)$ the subspace of $\mathscr{D}([0,1],\R^d)$ with finite $\alpha_{p\var}$ distance to $0$. Conventionally, write
	\[
	\alpha_\infty((X,\phi), (\overline{X},\overline{\phi})):= \lim_{\delta\to 0}\sigma_\infty(X^{\phi,\delta}, \overline{X}^{\overline{\phi},\delta})
	\]
\end{definition}
\begin{remark}\label{lm:a_infty-SM1}
If $\phi$ is the linear path function, then $\alpha_\infty$ induces the $\mathcal{SM}_1$ topology on the space $D([0,1],\R^d)$.
\end{remark}

\subsection{Background on rough differential equations (RDEs)}

\begin{definition}[Forward RDE]
We say that $X$ is a solution to the \emph{forward RDE} $\ds X_t= b(X)^-_t \ds W_t$ if
\[
X_t = X_0+ \int_0^t b(X_{s-}) \ds W_s,
\]
where the integral above denotes a limit of Riemann-Stieltjes sums with $b(X(s-))$ evaluated at the left limit points of the partition intervals:
\[
\int_0^t b(X_{s-})\ds W_s = \lim_{|\mathcal{P}|\to 0} \sum_{[s,s']\in \mathcal{P}} b(X_{s-})(W_{s'}-W_s),
\]
where the $\mathcal{P}$ are partitions of $[0,t]$ into intervals, and $|\mathcal{P}|$ is the size of the longest interval. 
\end{definition}

\begin{remark} 
We make  the following two observations.
\begin{itemize} 
    \item If $W$ has finitely many jumps at times $0 <t_1 <\ldots <t_n \le 1$, then the forward solution of $\ds X_t= b(X)_t^- \ds W_t, X_0=x$ can be obtained by solving the canonical RDE on each of the intervals on which $W$ is continuous, i.e., $[0,t_1), [t_1,t_2), \ldots, [t_n,1)$, and requiring that at jump times $t_k$, $k=1,\ldots, ,n$:
    \[
    X_{t_k} = X_{t_k-}+ b(X_{t_{k-}})(W_{t_k}- W_{t_{k-}}).
    \]
    \item If in \eqref{DS}, we write $X^{(m)}_t:= x^{(m)}_{\lfloor mt \rfloor}$ and  $V^{(m)}_t = \lfloor tm \rfloor/m$, then the first equation of \eqref{DS} is nothing but the forward RDE 
\[
\ds X^{(m)}_t = a_m(X^{(m)}_t)^- \ds V^{(m)}_t + b_m(X^{(m)}_t)^- \ds W^{(m)}_t, \qquad X^{(m)}_0 = x^{(m)}_0 \in \R^d.
\]
\end{itemize}
\end{remark}

\subsubsection{Young's integral and 
Marcus SDE}\label{sec:Marcus}
For the purpose of defining Young's integral for driving functions with finite $q$-variation, $q\ge1$, we fix $x\in C^{p\var}([0,1],\R^d)$ and $y\in C^{q\var}([0,1],\R^{e\times d})$ with $\theta= \frac{1}{p}+\frac{1}{q}>1$. 

\begin{definition}[Young's integral \cite{young1936}]
There exists a sequence $x^n\in C^{1\var}([0,T],\R^d)$ such that $x^n\to x$ in $C^{p\var}([0,T],\R^d)$ and a sequence $y_n\in C^{1\var}([0,T],\R^{e\times d})$ such that $y_n\to y$ in $C^{q\var}([0,T],\R^{e\times d})$. For every $s<t$, the limit of $\int_s^t y^n(u)dx^n(u)$ exists, which we denote by $\int_s^t y(u)dx(u)$ the \emph{Young's integral} of $y$ against $x$ on the interval $[s,t]$. The Young's integral does not depend on the choices of sequences $x^n$ and $y^n$, and we have
\[
\left|\int_s^t y(u)dx(u)-y(s)(x(t)-x(s))\right|\le \frac{1}{1-2^{1-\theta}}\|x\|_{p\var;[s,t]}\|y\|_{q\var;[s,t]}.
\]
\end{definition}

We record the following Young-L\'oeve estimate, which will be useful for our purposes.

\begin{proposition}[\cite{friz_victoir_2010} Theorem 6.8]\label{prop:integral_p-var}
Let $x\in C^{p\var}([0,T],\R^d)$ and $y\in C^{q\var}([0,T],\R^{e\times d})$ with $\frac{1}{p}+\frac{1}{q}>1$. The integral path $t\to \int_0^t y(u)dx(u)$ is continuous with a finite $p$-variation and 
\[
\left\|\int_0^. y(u)dx(u)\right\|_{p\var,[s,t]}\le C\|x\|_{p\var;[s,t]}(\|y\|_{q\var;[s,t]} + \|y\|_{\infty;[s,t]})
\]
\end{proposition}

Solutions to Marcus rough differential equation (RDE) are defined via Young's integrals and the linear path function defined in Definition \ref{def:X^phi,delta}.

\begin{definition}
The solution of a Marcus RDE
\[
\ds X= b(X)\diamond \ds W, \qquad X(0)= x_0
\]is obtained by
\begin{itemize}
    \item first solving the continuous RDE $\ds\tilde{X}= b(\tilde{X}) \ds W^{\phi,1}$, with $\phi= l_d$ being the linear path function on $\R^{d}$;
    \item then the c\`adl\`ag solution path $X[0,1]\to\R^d$ is given by $X(t)=\tilde{X}(\tau(t))$. Recall that $\tau:[0,1]\to[0,1+r] \text{ given by } \tau(t)=t+\sum_k\delta r_k 1_{\{t_k\le t\}}$.
\end{itemize}
\end{definition} 

\subsection{Example of a Marcus RDE}\label{sec:example_Marcus}
Here is an example of the solution of a rough differential equation in the Marcus sense \cite{chevyrev2020superdiffusive}.
Let $\theta>0$ and $W^{(m)}:[0,1] \to \R$  be the deterministic process which are $0$ on $[0,1/2]$, $\theta$ on $[1/2+1/m, 1]$ and linear on $[1/2, 1/2+1/m]$. Denote by $X^{(m)}_t = (x^{(m)}_t, y^{(m)}_t)$  the solution to the ordinary differential equation
\begin{equation*}
    \begin{pmatrix}
\ds x^{(m)}_t\\
\ds y^{(m)}_t
\end{pmatrix} = 
 \begin{pmatrix}
- y^{(m)}_t\\
 x^{(m)}_t
\end{pmatrix} \ds W^{(m)}_t, \quad 
 \begin{pmatrix}
x^{(m)}_0\\
y^{(m)}_0
\end{pmatrix} = 
 \begin{pmatrix}
 1\\
 0
\end{pmatrix}
\end{equation*}
It is not hard to see that that $X^{(m)}_t = (\cos{W^{(m)}_t}, \sin{W^{(m)}_t})$. Therefore, we have that pointwise
\[
W^{(m)}_t \to W_t = \theta 1_{t\in (\frac{1}{2},1]}\, \text{ and } \, X^{(m)}_t \to X_t =
\begin{pmatrix}
x_t\\
y_t
\end{pmatrix}
=
\begin{pmatrix}
1 \\
0
\end{pmatrix}1_{t\in [0,\frac{1}{2}]} +
\begin{pmatrix}
\cos{\theta}\\
\sin{\theta}
\end{pmatrix})1_{t\in (\frac{1}{2},1]}.
\]
If $\theta\neq 0$, $(X_t)_{t\in[0,1]}$ fails to be the right-hand side of the limiting forward RDE since
\[
\int_0^1 \begin{pmatrix}
-y_{s-}\\
x_{s-}
\end{pmatrix} \ds W_s =
\lim_{|\mathcal{P}|\to 0} \sum_{[s,s']\in \mathcal{P}} \begin{pmatrix}
-y_{s-}\\
x_{s-}
\end{pmatrix}(W_{s'}-W_s)=
\begin{pmatrix}
-0\\
1
\end{pmatrix} \times \theta = 
\begin{pmatrix}
0\\
\theta
\end{pmatrix}\neq X_1-X_0.
\]
Effectively, $(X_t)_{t\in[0,1]}$ is the solution to the Marcus RDE in the limit.

\section{Homogenization: The Rigorous Version and the Proof of Theorem~\ref{thm:main}}\label{sec:homogenization}

When talking about convergence of solutions to RDEs, we need to specify the sense of integration. In view of the example in Subsection \ref{sec:example_Marcus}, it is not enough to look at the solution $X$ as an element of $D([0,1],\R^d)$ -- it has to be coupled with jumps of the driving function. Let us consider the driver-solution space $D([0,1],\R^{r+d})$ and introduce a new path function on $\R^{r+d}$.

\begin{definition}
Consider $b\in C(\R^d, \R^{d\times r})$, $a\in C(\R^d,\R^d)$. For $x\in \R^d$ and $\Phi\in C^{1\var}([0,1],\R^r)$, let $\pi_{a,b}[x;\Phi]\in C^{1\var}([0,1],\R^d)$ denote the solution $\Pi$ of the equation
\[
\ds\Pi = a(\Pi)\ds t+ b(\Pi) \ds\Phi, \qquad \Pi(0)=x.
\]
For the coefficient $b$, a pair of admissible points  $(w_1,x_1), (w_2,x_2)\in J_{a,b}\subset \R^{r+d}$ is in the space
\[
J_{a,b}= \{\big((w_1,x_1), (w_2,x_2)\big): w_1,w_2\in\R^d, \pi_b[x_1; l_d(w_1,w_2)](1)=x_2\}.
\]
We define the path function $\phi_{a,b}$ on $J_{a,b}$ by
\[
\phi_{a,b} ((w_1,x_1),(w_2,x_2))= \big( l_d(w_1,w_2)(t), \pi_b[x_1; l_d(w_1,w_2)](t) \big).
\]
\end{definition}

Next, we first state the precise assumptions for Theorem \ref{thm:main} in the main paper. The required topology for the coefficients $a_m, b_m$ in \eqref{DS_1var} is as follows. For $\tilde{\gamma}>0$, $n_1,n_2\in \mathbb{N}_+$, denote by $C^{\tilde{\gamma}}(\R^{n_1},\R^{n_2})$ the space of functions $f: \R^{n_1} \to \R^{n_2}$ such that 
\footnote{
For any $n\in\mathbb{N}$, $\mathbf{j}= (j_1, \ldots, j_n)\in \mathbb{N}^n_0$ is a multi-index with $|\mathbf{j}|=\sum_{i=1}^n j_i$ and the higher-order partial derivative is defined by $\partial^\mathbf{j} = (\partial/\partial_{x_1}^{j_1}) \cdots (\partial/\partial_{x_n}^{j_n}) $.
}
\[
\|f\|_{C^{\tilde{\gamma}}} = \max_{|\mathbf{j}|=0,\ldots,\lfloor\gamma'\rfloor}\|\partial^\mathbf{j} f\|_\infty + \sup_{x,y\in\R^m}\max_{|\mathbf{j}|=\lfloor\tilde{\gamma}\rfloor} \frac{|\partial^\mathbf{j} f(x) - \partial^\mathbf{j} f(y)|}{|x-y|^{\tilde{\gamma}-\lfloor\tilde{\gamma}\rfloor}}<\infty.
\]
In particular, if $\tilde{\gamma}>1$, we have
$
\|f\|_\infty \le \|f\|_{C^{\tilde{\gamma}}}$ and $\|f\|_{\mathrm{Lip}} \le \|f\|_{C^{\tilde{\gamma}}}$.

\begin{ass}\label{assmp:linf}
Suppose that $a\in C^{\gamma_1}(\R^d,\R^d)$, $b\in C^{\gamma_2}(\R^d,\R^{d\times r})$ for some $\gamma_1 > 1, \gamma_2 > \alpha$. In addition, we assume that
 \[
 \limm a_m= a \text{ in }C^{\gamma_1}(\R^d,\R^d),\, \limm b_m= b \text{ in }C^{\gamma_2}(\R^d,\R^{d\times r})\, \text{ and } \,\limm x^{(m)}_0 = x_0.
 \]
\end{ass}

Assumption \ref{assmp:linf} requires that the difference between the gradients of coefficients $\| \nabla (a_m - a) \|_\infty$  goes to zero  and also their H\"older constant  $\frac{|\nabla (a_m - a)(x) -  \nabla (a_m - a)(y)| }{|x-y|^{\epsilon}}$
goes to zero  uniformly on $x, y \in \R^m$  for some $\epsilon > 0$. 
This translates to the assumptions that the loss is second differentiable and its Hessian is $\epsilon$-H\"older, which are reasonable assumptions.
Similarly, for the coefficients appearing before the $\alpha$-stable process (for $\alpha \in (1,2)$ ), it is required that 
$\| \nabla (b_m - b) \|_\infty$  goes to zero  and 
$\sup_{x, y \in \R^m} \frac{|\nabla (b_m - b)(x) -  \nabla (b_m - b)(y)| }{|x-y|^{\epsilon'}}$ goes to zero for some $\epsilon' > \alpha - 1$.
In the setup for MPGD \eqref{eq_ouralg}  where $a_m = - \hat{\mathcal{R}}(x, S_n) = a$ (independent of $m$) and the $b_m$ is  $-\mu \  \rm{diag}(x)$ or $\sigma I$, it is easy to see that these assumptions are satisfied.

In terms of how general the assumption is, we note that Assumption \ref{assmp:linf} can cover more general situations; e.g., the case where  is an empirical loss (taking $m=n$, the number of training samples) and $a$ is the population loss, and therefore our framework can be adapted for analysis in other settings as well.

Theorem \ref{thm:main} is an extension of Theorem 4.1  in \cite{gottwald2021simulation} (with the $\epsilon$ there equal $1/m$) to the case where the coefficients $a$, $b$ are dependent on the scaling parameter $m$. In particular, the proof of Theorem 4.1 in \cite{gottwald2021simulation} relies on verifying the hypotheses of Theorem 2.6 in \cite{chevyrev2020superdiffusive} for the particular example of the observable and Thaler map constructed for the MPGD \eqref{eq_ouralg}.

We now prove a more general theorem that extends Theorem 2.6 in \cite{chevyrev2020superdiffusive} and includes Theorem \ref{thm:main}  as a  special case, stating all the needed assumptions. Applying this theorem together with the hypotheses verification  in the proof of Theorem 4.1 in \cite{gottwald2021simulation} completes the proof of  Theorem \ref{thm:main}.


\begin{theorem}[A more general version of Theorem~\ref{thm:main}]\label{thm:main_app}
Let $\alpha\in (1,2)$, $\gamma_1 > 1$ and $\gamma_2>\alpha$. Let $v\in L^\infty(Y,\R^d)$ be H\"older such that $\int v\ds\mu=0$, where $\mu$ is the unique ergodic invariant probability measure of $\ T:Y\to Y$. We focus on the case where $T$ exhibits superdiffusive behavior, which means that there exists a $r$-dimensional L\'evy process $\mathrm{L}$ and
\[
W_n(t) = n^{-1/\alpha}\sum_{j=0}^{\lfloor nt \rfloor -1} v\circ T^j \overset{\mathrm{(d)}}{\to} \mathrm{L}_t \text{ in } D([0,1],\R^r) \text{ under the } \mathcal{SM}_1 \text{ topology (Definition \ref{def:sm1}) } 
\]
as $n\to \infty$.

If, in addition,
\begin{itemize}
    \item $\lim_{n\to \infty}a_n= a$ in $C^{\gamma_1}(\mathbb{R}^d,\mathbb{R}^d)$ for some $a\in C^{\gamma_1}(\mathbb{R}^d,\mathbb{R}^d)$,
    \item $\lim_{n\to \infty}b_n= b$ in $C^{\gamma_2}(\mathbb{R}^d,\mathbb{R}^{d\times r})$ for some $b\in C^{\gamma_2}(\mathbb{R}^d,\mathbb{R}^{d\times r})$;
    \item $\lim_{n\to \infty} x_n= x$,
    \item $\|W_n\|_{p\var}$ is tight for all $p>\alpha$ and $\sum_{t}|W_n(t)-W_n(t-)|^2 \to 0$ a.s.,  where the sum is taken over all jump times of $W_n$,
\end{itemize}
then for the forward RDE $\ds X_n = a_n(X_n)^- \ds V_n + b_n(X_n)^- \ds W_n, X_n(0)= x_n$ where $V_n(t)= \frac{\lfloor tn \rfloor}{n}$, we have:
\[
((W_n, X_n), l_{r+d}) \overset{\mathrm{(d)}}{\to} ((\mathrm{L},X),\phi_{a,b}) \text{ in } (\mathscr{D}^{p\var}([0,1],\R^{r+d}), \alpha_{p\var}) \text{ as } n\to \infty
\]
for all $p>\alpha$, where $(X(t))_{t\ge 0}$ is the solution of the Marcus differential equation 
\[
\ds X= a(X)\ds t + b(X)\diamond \ds \mathrm{L} \text{ with } X(0)= x.
\]
\end{theorem}

\begin{remark}
With the choices of $v$ and $T$ in Section \ref{sec:MPGD}, the assumptions in Theorem \ref{thm:main_app} are satisfied (see the proof in \cite{gottwald2021simulation}).
\end{remark}

Before proving Theorem~\ref{thm:main_app}, we start by proving the following lemma. 

\begin{lemma}\label{lm:p-var_path}
	Let $X\in D([0,1],\R^d)$, $p\ge 1$ such that $\|X\|_{p\var}<\infty$. Then for any path function $\phi$, 
	\[
	\|X^{\phi,\delta}\|_{p\var}= \|X^\phi\|_{p\var}\ge \|X\|_{p\var} = \|X^{l_d}\|_{p\var}.
	\]
\end{lemma}

\begin{proof} It is obvious from the scaling invariance of  of the $p$-variation that~$\|X^{\phi,\delta}\|_{p\var}= \|X^\phi\|_{p\var}$. For the second inequality, using again the definition,
	\begin{align*}
	&\|X^\phi\|_{p\var} \\
	&=\|\hat{X}(\cdot(1+r))\|_{p\var}\\
	&= \sup_{0=t_0<t_1<\ldots<t_k=1} \left(\sum_{j=1}^k |\hat{X}(t_j(1+r))- \hat{X}(t_{j-1}(1+r))|^p\right)^{1/p}\\
	&\ge \sup_{0=t_0<t_1<\ldots<t_k=1, t_j(1+r)=\tau(s_j) \text{ for some }s_j\in[0,1]}\left(\sum_{j=1}^k |\hat{X}(t_j(1+r))- \hat{X}(t_{j-1}(1+r))|^p\right)^{1/p}\\
	&= \sup_{0=s_0<s_1<\ldots<s_k=1}\left(\sum_{j=1}^k |X(s_j)- X(s_{j-1})|^p\right)^{1/p} = \|X\|_{p\var},
	\end{align*}
	where the last equality follows from the fact that if $c$ lies in an interval $[a,b]$ then~$(c-a)^p+(b-c)^p\le (b-a)^p$ for $p\ge 1$. This completes the proof.
\end{proof}

As a first step for proving Theorem~\ref{thm:main_app}, we aim to prove a deterministic variant of Theorem \ref{thm:main_app}. In order to make the arguments more concise, let us neglect the first time-derivative term (with coefficients $a$ and $a_n$, for which the treatment is the same by considering the jump process $V_n:t\mapsto \lfloor nt\rfloor/n$) and consider the following. 

\begin{theorem}\label{thm:main_det}
Assume that $p\in(1,2)$ and $\{W_n\}_{n\ge 1}$ is a sequence in $D^{p\var}([0,1],\R^r)$ with finitely many jumps. Suppose that $\gamma>p$ and
\begin{itemize}
    \item $\lim_{n\to \infty}b_n= b$ in $C^\gamma(\mathbb{R}^d,\mathbb{R}^{d\times r})$ for some $b\in C^\gamma(\mathbb{R}^d,\mathbb{R}^{d\times r})$;
    \item $d_{\mathcal{SM}_1}(W_n,W)\to 0$ and $\sum|W_n(t)-W_n(t-)|^2\to 0$, where the sum is taken over all jump times of $W_n$.
\end{itemize}
Let $X_n$ be the solution of the forward RDE $\ds X_n = b_n(X_n)^-\ds W_n$ with $X(0)=x_n$. Let $X$ be the solution of the Marcus RDE~$\ds X= b(X)\diamond \ds W, X(0)=x$. Then it holds that
\[
((W_n,X_n), l_{r+d}) \to ((W,X),\phi_b) \text{ in } (\mathscr{D}([0,1],\R^{r+d}), \alpha_{p\var}) \text{ as } n\to\infty.
\]
\end{theorem}


\subsection{Proof of Theorem \ref{thm:main_det}}

We will use the following lemma in the proof of Theorem~\ref{thm:main_det}. 

\begin{lemma}[\cite{chevyrev2020superdiffusive} Lemma 3.6]\label{lm:jumps}
Suppose that $W\in D([0,1],\R^r)$ has finitely many jumps. Let $X,\tilde{X}\in D([0,1],\R^d)$ be given by
\[
\ds X= b(X)^- \ds W, \quad \ds \tilde{X} = b(\tilde{X})\diamond\ds W, \quad X(0)=\tilde{X}(0)=x.
\]
Then
\[
\|X-\tilde{X}\|_{p\var} \le K\|b\|_{\mathrm{Lip}}\|b\|_\infty \sum_t|W(t)-W(t-)|^2,
\]
where $K$ depends on $\|b\|_{C^\gamma}, \|W\|_{p\var}, \gamma, p$, and the sum is over all jump times $t$ of $W$.
\end{lemma}

We are now ready to prove Theorem~\ref{thm:main_det}.

\begin{proof}[Proof of Theorem~\ref{thm:main_det}]
Let $\tilde{X}_n$  be the solution to the Marcus RDE
\[
\ds \tilde{X}_n =  b(\tilde{X}_n)\diamond \ds W_n, \tilde{X}_n(0)=x_n.
\]
The continuity of the solution map for generalised geometric RDEs gives that
\begin{equation}\label{eq:1}
\alpha_{p\var}(((W_n,\tilde{X}_n),\phi_b), ((W,X),\phi_b))=0.
\end{equation}
Then let $\overline{X}_n$ be the solution to the Marcus RDE
\[
\ds\overline{X}_n = b_n(\overline{X}_n) \diamond \ds W_n,\quad \overline{X}_n(0)=x_n.
\]
On any subinterval $[s_1,s_2]$, we can compare solutions to the Marcus RDE with $b_n$ and $b$ by using Proposition \ref{prop:integral_p-var}:
\begin{align*}
    &\|\overline{X}_n-\tilde{X}_n\|_{p\var}\\ =
    &\left\|\int_0^.(b_n(\overline{X}_n)-b(\tilde{X}))\ds W^{l_r}  \right\|_{p\var}\\ \le
    & \left\|\int_0^.(b_n(\overline{X}_n)-b(\overline{X}_n))dW^{l_r}\right\|_{p\var} + \left\|\int_0^.(b(\overline{X}_n)-b(\tilde{X}_n))dW^{l_r}\right\|_{p\var} \\ \le
    & \|b_n-b\|_\infty\|W^{l_r}\|_{p\var} + C_1(\|b\|_{\text{Lip}}\|\overline{X}_n-\tilde{X}_n\|_{p\var}+|b(\overline{X}_n)-b(\tilde{X}_n)|(s_1))\|W^{l_r}\|_{p\var} \\ =
    & \|b_n-b\|_\infty\|W\|_{p\var} + C_1(\|b\|_{\text{Lip}}\|\overline{X}_n-\tilde{X}_n\|_{p\var}+|b(\overline{X}_n)-b(\tilde{X}_n)|(s_1))\|W\|_{p\var},
\end{align*}
where the last inequality is due to Lemma \ref{lm:p-var_path}.
If we choose a subdivision $0=s_0<s_1<\ldots<s_{n-1}<s_n=1$ of the interval $[0,t]$ such that for any $i=0,1,\ldots,n-1$, \[
C_1\|b\|_{\text{Lip}}\|W\|_{p\var,[s_i,s_{i+1}]} \le c< 1,
\]
then on the interval $[0,s_1]$,
\begin{align*}
\|\overline{X}_n-\tilde{X}_n\|_{p\var,[0,s_1]}
&\le \frac{ \|b_n-b\|_\infty \|W\|_{p\var,[0,s_1]}+ C_1\|W\|_{p\var,[0,s_1]}|b(\overline{X}_n)-b(\tilde{X}_n)|(0)}{1- C_1\|b\|_{\text{Lip}}\|W\|_{p\var,[0,s_1]} }\\
&\le  \|b_n-b\|_\infty \|W\|_{p\var,[0,s_1]}/(1-c).
\end{align*}
Similarly on the interval $[s_1,s_2]$,
\begin{align*}
&\|\overline{X}_n - \tilde{X}_n\|_{p\var,[s_1,s_2]}\\ \le
& \frac{ \|b_n-b\|_\infty \|W\|_{p\var,[s_1,s_2]} + C_1\|W\|_{p\var,[s_1,s_2]}|b(\overline{X}_n) - b(\tilde{X}_n)|(s_1)}{1- C_1\|b\|_{\text{Lip}}\|W\|_{p\var,[s_1,s_2]}}\\ \le
& \left( \|b_n-b\|_\infty \|W\|_{p\var,[s_1,s_2]} + C_1\|W\|_{p\var,[s_1,s_2]}\|b\|_{\text{Lip}}\|\overline{X}_n - \tilde{X}_n\|_{p\var,[0,s_1]} \right)/(1-c) \\ \le
& \|b_n-b\|_\infty \|W\|_{p\var,[s_1,s_2]}/(1-c) + C_1\|b\|_{\text{Lip}}\|b_n-b\|_\infty\|W\|_{p\var,[0,s_1]}\|W\|_{p\var,[s_1,s_2]} /(1-c)^2.
\end{align*}
It is not hard to obtain by induction that there exists $C(b,W)>0$ such that
\[
\|\overline{X}_n - \tilde{X}_n\|_{p\var,[0,1]}\le C(b, W)\|b_n-b\|_\infty.  
\]
Therefore
\begin{equation}\label{eq:2}
  \alpha_{p\var}(((W_n,\tilde{X}_n),\phi_b), ((W_n,\overline{X}_n),\phi_{b_n}))=0.  
\end{equation}

Since $b_n\to b$ in $C^\gamma(\R^d,\R^{d\times r})$, $\|b_n\|_{\mathrm{Lip}}$ and $\|b_n\|_\infty$ are uniformly bounded independent of $n$. Then it follows from Lemma \ref{lm:jumps} that
\[
\|X_n- \overline{X}_n\|_{p\var}\le K\|b_n\|_{\mathrm{Lip}}\|b_n\|_\infty\sum_t\|W_n(t)-W_n(t-)\| \to 0 \text{ as }n\to\infty.
\]
Then the same argument in \cite{chevyrev2020superdiffusive} gives that for any $p'>p$,
\begin{equation}\label{eq:3}
\lim_{n\to\infty} \alpha_{p'\var}(((W_n,\overline{X}_n),l_{r+d}), ((W_n,X_n),l_{r+d}))=0.
\end{equation}
Recall also from the \cite{chevyrev2020superdiffusive} that
\begin{equation}\label{eq:4}
\lim_{n\to\infty}\alpha_{p'\var} (((W_n,\overline{X}_n),\phi_{b_n}), (W_n,\overline{X}_n),l_{r+d}))=0.
\end{equation}
Finally Theorem \ref{thm:main_det} follows from combining \eqref{eq:1}, \eqref{eq:2}, \eqref{eq:3} and \eqref{eq:4}. This completes the proof.
\end{proof}

\subsection{Proof of Theorem \ref{thm:main_app}}

\begin{proof}[Proof of Theorem \ref{thm:main_app}]
 If we write $V_n(t)=\lfloor tn\rfloor/n$, the assumptions of Theorem \ref{thm:main_app} and the same treatment in \cite{chevyrev2020superdiffusive} implies that up to subtracting a subsequence, $(V_n,W_n)\to (\mathrm{id},\mathrm{L})$ a.s. in $\alpha_{p\var}$ for any $p>\alpha$. Then it follows from Theorem \ref{thm:main_det} that for $p'>p$ and along each subsequencial limit of $(V_n,W_n)$ as $n\to\infty$,
 \[
 ((W_n,X_n),l_{r+d}) \to ((\mathrm{L},X),\phi_{a,b}) \, \text{ in }\, (\mathcal{D}([0,1],\R^{r+d}),\alpha_{p'\var}).
 \]
 Therefore $((W_n,X_n),l_{r+d}) \overset{\mathrm{(d)}}{\to} ((\mathrm{L},X),\phi_{a,b}) \, \text{ in }\, (\mathcal{D}([0,1],\R^{r+d}),\alpha_{p\var})$ for any $p>\alpha$. This completes the proof.
 
\end{proof}

\section{Generalization Bound: Proof of Theorem~\ref{thm:gen}}
\label{sec:proofgen}

To begin with, we need a geometric regularity assumption over the trajectory of the multiscale perturbed gradient flow, which is common for random fractal processes given as solutions to stochastic differential equations; see \cite{hodgkinson2021generalization}. This assumption ensures that the box-counting (Minkowski) dimension coincides with the Hausdorff dimension of the trajectory.

For any $x\in \R^d$, let $(X_t)_{t\in [0,1]}$ be a solution to the stochastic differential equation \eqref{eq:sde} started from $x$, and let $\mathcal{A}$ be the infinitesimal generator of $(X_t)_{t\in [0,1]}$ (let us take it for granted that it exists) defined by
\[
\mathcal{A}u(x) = \lim_{t\to 0}\frac{\mathbb{E}_x[u(X_t)]-u(x)}{t} \, \text{ for any }\, u\in C^\infty_c(\mathbb{R}^d).
\]
Let $q:\mathbb{R}^m\times\mathbb{R}^m \to \mathbb{C}$ be the \emph{symbol} of $\mathcal{A}$ such that $
\mathcal{A}u(x) = -\int_{\mathbb{R}^m} e^{i\scp{\xi}{x}}q(x,\xi)\hat{u}(\xi) \ds \xi$,
(informally saying, $-q(x,D)u(x)$), where~$\hat{u}(\xi)= \int_{\mathbb{R}^m} e^{-i\scp{\xi}{x}}u(x) \ds x$ is the Fourier transform of $u$.

\begin{ass}\label{assmp:traj_regularity}
For almost every $\mathcal{X}$, there exists a finite Borel measure $\mu$ on $\mathcal{X}$ and $\rho >0$ such that $
C_\rho \coloneqq \inf_{0 < r < \rho, x \in \mathcal{X}} \frac{\mu(B_r(x))}{r^\alpha \mu(\mathcal{X})} > 0$
for $\mu$-almost every $x$. 
\end{ass}

If $(X_t)_{t\ge 0}$ is a solution with $X_0=0$ to $\ds X_t = b \diamond \ds \mathrm{L}^\alpha_t$ for fixed $b\in \R^{d\times d}$ and a symmetric $d$-dimensional L\'evy process $\mathrm{L}^\alpha$, then $(X_t)_{t\ge 0}$ is also a $\alpha$-stable L\'evy process with symbol independent of $x$:
\[
\psi_{X}(\xi) = |\xi^T b|^\alpha.
\]

We also need the following statistical regularity assumption, as discussed in the paragraph before Theorem \ref{thm:gen}. 

\begin{ass}\label{assmp:haus_dim}
There exists $x^\star \in \mathbb{R}^d$ such that $a(x^*)=0$. Let $(X^*_t)_{t\ge 0}$ to the solution of $\ds X_t = b(x^*) \diamond \ds \mathrm{L}^\alpha_t$ and write $\tilde{q}(x,\xi):= q(x,\xi)-\psi_{X^*}(\xi)$. For almost-every $S_n$:
\begin{itemize}
    \item $b(x^\star)$  is positive definite;
    \item $|\partial^\mathbf{j}_x \tilde{q}(x,\xi)|\le \Phi_\mathbf{j}(x)(1+\kappa_0|\xi|^\alpha) \text{ with }\mathbf{j}\in \mathbb{N}_0^m,\, |\mathbf{j}|\le m+1 \text{ for some }\Phi_\mathbf{j} \in L^1(\mathbb{R}^m);$
    \item $q(x,0)=0 \text{ and }\|\Phi_0\|_\infty<\infty.$
\end{itemize}
\end{ass}

Some remarks on Assumption \ref{assmp:haus_dim} are now in place.  It is natural to assume the existence of $x^*$ such that $a(x^*) = 0$ (since $a = - \nabla \hat{\mathcal{R}}(x, S_n)$ in our MPGD setup). For the first point, the positive definiteness of $b(x^*)$ can be satisfied by simply choosing $b$ to be the identity map. For the second point, recall that $\tilde{q}(x,\xi) = q(x, \xi) - \psi_{x^*}(\xi)$ and we require that:
$|\partial^\mathbf{j}_x \tilde{q}(x,\xi)|\le \Phi_\mathbf{j}(x)(1+\kappa_0|\xi|^\alpha) \text{ with }\mathbf{j}\in \mathbb{N}_0^m,\, |\mathbf{j}|\le m+1 \text{ for some }\Phi_\mathbf{j} \in L^1(\mathbb{R}^m)$.

Here $X^*$ is an $\alpha$-stable L\'evy process, and its characteristic function $\psi_{x^*}(\xi)$ is given by $|\xi^T b(x^*)|^\alpha$ (see Eq. \eqref{eq_charexp}). Moreover, $q(x, \xi)$ is nothing but \eqref{eq_charexp} with the $b, \Sigma, \mu$  depending on $x$. It is not hard to see that the above assumption holds if the $b, \Sigma, \mu$ depend smoothly on $x$. The third point is equivalent to saying that the solution to the SDE exists almost surely on infinite time interval (see \cite{schilling_conservativeness}), which is the case for the perturbed gradient flow with respect to second differentiable loss.

With all these ingredients in place, we now prove Theorem \ref{thm:gen}. 
\begin{proof}
We observe that the ellipticity condition holds since $b(x^\star)$ is positive  definite.:
\[
1+ \psi_{X^*}(\xi) \ge \gamma_1 (1+\kappa_0|\xi|^\alpha) \text{ for some } \gamma_1, \,\kappa_0>0.
\]

One can also check that $
 \alpha :=  \inf\{\lambda\ge 0: \lim_{|\xi|\to\infty}\frac{|\psi_{X^*}(\xi)|}{|\xi|^\lambda}=0\}.$ Under Assumption \ref{assmp:haus_dim}, it follows from \cite[Theorem 4]{schilling1998feller} that
\[
\dim_{\mathrm{H}}(\mathcal{X}) \le \alpha \quad \text{ almost surely,}
\]
where $\dim_{\mathrm{H}}$ denotes the Hausdorff dimension.

Then, under Assumption \ref{assmp:traj_regularity}, the result follows from \cite[Theorem 1]{hodgkinson2021generalization} and \cite[Corollary 1]{hodgkinson2021generalization}. This completes the proof. 
\end{proof}

\section{Implicit Regularization: Proof of Theorem \ref{thm_implreg}}
\label{app_sec:implreg}

In this section, we provide  proof to Theorem \ref{app_sec:implreg}.

Let $r_1 = 1$, $r_2 = d$ as assumed in Theorem \ref{thm_implreg}. Upon the rescalings $\mu = \mu_0 \epsilon$ and $\sigma = \sigma_0 \epsilon$ with $\epsilon > 0$ a small parameter, the MPGD recursion Eq. \eqref{eq_ouralg} becomes:
\begin{align} \label{eq_mpgd_rescaled}
    \xn{k+1} &=  \xn{k} + f(\xn{k}) + \epsilon g(\xn{k}, y_k^{(1)}, y_k^{(2)}), 
\end{align}
for $k=0,1,2, \dots$, 
where 
\begin{align}
     f(\xn{k}) &= - \frac{1}{m}  \nabla \hat{\mathcal{R}}(\xn{k}) =: f_k^{(m)}, \\
     g(\xn{k}, y_k^{(1)}, y_k^{(2)}) &= -  \frac{\mu_0}{m^{\frac{1}{\alpha_1}}} v_{1}(y_k^{(1)})  \xn{k} + \frac{\sigma_0}{m^{\frac{1}{\alpha_2}}} v_{2}(y_k^{(2)}) =: g_k^{(m)},
\end{align}
with the $y^{(i)}$ satisfying $y^{(i)}_{k+1} = f^{(i)}(y_k^{(i)})$ for $i=1,2$.

Consider the following hierarchy of recursive equations. For $k=0,1,2,\dots$:
\begin{align}
    \overline{x}_{k+1}^{(m)} &= \overline{x}_{k}^{(m)} - \frac{1}{m} \nabla \hat{\mathcal{R}}(\overline{x}_k^{(m)}), \label{eq_hie1} \\
    \phi_{k+1}^{(m)} &= J_k^{(m)} \phi_{k}^{(m)} + g_k^{(m)}, \label{eq_hie2} \\
    \varphi_{k+1}^{(m)} &= J_k^{(m)} \varphi_{k}^{(m)} + \frac{1}{2} \sum_{i,j=1}^d \frac{\partial^2  f_k^{(m)}} {\partial x^i \partial x^j}(\overline{x}_k^{(m)}, \phi_k^{(m)})  [\phi_k^{(m)}]^i [\phi_k^{(m)}]^j + \sum_{i=1}^d \frac{\partial g_k^{(m)}}{\partial x^i}(\overline{x}_k^{(m)}) [\phi_k^{(m)}]^i, \label{eq_hie3}
\end{align}
with the initial conditions $\overline{x}_0^{(m)} = x_0^{(m)}$, $\phi_0^{(m)} = \varphi_0^{(m)} = 0$, and
\begin{align}
    J_k^{(m)} &:= I + f'(\overline{x}_k^{(m)}) = I - \frac{1}{m}  \nabla^2 \hat{\mathcal{R}}(\overline{x}_k^{(m)}).
\end{align}

Informally, the processes $\overline{x}^{(m)}, \phi^{(m)}$ and $\varphi^{(m)}$ constitute the zeroth-, first- and second-order terms in a pathwise Taylor expansion
of $x^{(m)}$ about $\epsilon = 0$, i.e., for $k=0,1,2, \dots$, $x_k^{(m)} = \overline{x}_k^{(m)} + \epsilon \phi^{(m)}_k + \epsilon^2 \varphi^{(m)}_k + \mathcal{O}(\epsilon^3)$, as $\epsilon \to 0$. 

Now we are going to compute a second-order Taylor expansion
of $\hat{\mathcal{R}}(x_k^{(m)})$ about $\epsilon = 0$. This is the content of the following lemma, which is adapted from Theorem 9 in \cite{lim2021noisy} to our setting.

\begin{lemma} \label{lem_aux}
Under the assumptions of Theorem \ref{thm_implreg}, we have:
\begin{align}
    \hat{\mathcal{R}}(x_k^{(m)}) &= \hat{\mathcal{R}}(\overline{x}_k^{(m)}) + \epsilon \nabla \hat{\mathcal{R}}(\overline{x}_k^{(m)})^T \phi_k^{(m)} + \epsilon^2 \left(\hat{\mathcal{R}}(\overline{x}_k^{(m)})^T \varphi_k^{(m)} + \frac{1}{2} (\phi_k^{(m)})^T \nabla^2 \hat{\mathcal{R}}(\overline{x}_k^{(m)}) \phi_k^{(m)}   \right) \nonumber \\
    &\ \ \  \ \ + \mathcal{O}(\epsilon^3),
\end{align}
as $\epsilon \to 0$, where the $(\overline{x}_k^{(m)}, \phi^{(m)}, \varphi^{(m)})$ satisfy Eq. \eqref{eq_hie1}-\eqref{eq_hie3}.
\end{lemma}

We now prove Theorem \ref{thm_implreg}.
\begin{proof}[Proof of Theorem \ref{thm_implreg}]
Let us fix the $m$ and $S_n$. To prove Theorem \ref{thm_implreg}, we need to compute $\mathbb{E} \hat{\mathcal{R}}(x_k^{(m)})$, where the expectation is with respect to the randomness in the $y^{(i)}_0$. Since, $\mathbb{E}[v_1(y^{(1)}_k)] = \mathbb{E}[v_2(y^{(2)}_k)] = 0  $ for all $k$ by assumption, we have $\mathbb{E} \phi_k^{(m)} = 0$ for all $k$, and applying Lemma \ref{lem_aux}:
\begin{align}
     \mathbb{E} \hat{\mathcal{R}}(x_k^{(m)}) &= \hat{\mathcal{R}}(\overline{x}_k^{(m)}) +  \epsilon^2 \left(\hat{\mathcal{R}}(\overline{x}_k^{(m)})^T \mathbb{E} \varphi_k^{(m)} + \frac{1}{2} \mathbb{E} \left[ (\phi_k^{(m)})^T \nabla^2 \hat{\mathcal{R}}(\overline{x}_k^{(m)}) \phi_k^{(m)} \right]  \right) + \mathcal{O}(\epsilon^3). \label{eq_expansion}
\end{align}

It remains to compute the $\mathbb{E} \varphi_k^{(m)}$ and $\mathbb{E} \left[ (\phi_k^{(m)})^T \nabla^2 \hat{\mathcal{R}}(\overline{x}_k^{(m)}) \phi_k^{(m)} \right]$ in the above expansion. 

Iterating Eq. \ref{eq_hie2}, we obtain
\begin{align}
    \phi_k^{(m)} &= \sum_{i=1}^k \left(\prod_{j=i}^{k-1} J_j^{(m)} \right) g_{i-1}^{(m)},
\end{align}
for $k=1,2,\dots$.
Similarly, iterating Eq. \ref{eq_hie3}, we obtain:
\begin{align}
    \varphi_k^{(m)} &= \sum_{i=1}^k \left(\prod_{j=i}^{k-1} J_j^{(m)} \right) \left[ \frac{1}{2} \sum_{p,q=1}^d \frac{\partial^2  f_{i-1}^{(m)}} {\partial x^p \partial x^q}(\overline{x}_{i-1}^{(m)})  [\phi_{i-1}^{(m)}]^p [\phi_{i-1}^{(m)}]^q + \sum_{l=1}^d \frac{\partial g_{i-1}^{(m)}}{\partial x^l}(\overline{x}_{i-1}^{(m)}) [\phi_{i-1}^{(m)}]^l  \right]. \label{eq_phi}
\end{align}

Now, we compute: 
$\frac{\partial g_{i-1}^{(m)}}{\partial x^l}(\overline{x}_{i-1}^{(m)}) =  - \frac{\mu_0}{m^{\frac{1}{\alpha_1}}} v_{1}(y_{i-1}^{(1)})$ for all $l$, and 
\begin{align}
    \frac{\partial^2  f_{i-1}^{(m)}} {\partial x^p \partial x^q}(\overline{x}_{i-1}^{(m)})  &= -\frac{1}{m} \frac{\partial^2 \nabla \hat{\mathcal{R}}}{\partial x^p \partial x^q}(\overline{x}_{i-1}^{(m)}).
\end{align}
Substituting the above formula into Eq. \eqref{eq_phi} and then taking expectation, we obtain, using the fact that $\mathbb{E} v_{1}(y_{i-1}^{(1)}) = 0$ for all $i$:
\begin{align}
    \mathbb{E} \varphi_k^{(m)} &= - \frac{1}{2m} \sum_{i=1}^k \left(\prod_{j=i}^{k-1} J_j^{(m)} \right)  \sum_{p,q=1}^d \frac{\partial^2 \nabla \hat{\mathcal{R}}} {\partial x^p \partial x^q}(\overline{x}_{i-1}^{(m)}) \cdot \mathbb{E}[ [\phi_{i-1}^{(m)}]^p [\phi_{i-1}^{(m)}]^q] =:  - \frac{1}{2} \lambda_k^{(m)},
\end{align}
with the $l$th component of:
\begin{equation}
    - \frac{1}{2m} \sum_{i=1}^k \sum_{r=1}^d \left[\prod_{j=i}^{k-1} J_j^{(m)} \right]^{lr}  \sum_{p,q=1}^d \frac{\partial^2 \nabla [\hat{\mathcal{R}}]^r} {\partial x^p \partial x^q}(\overline{x}_{i-1}^{(m)}) \cdot \mathbb{E}[ [\phi_{i-1}^{(m)}]^p [\phi_{i-1}^{(m)}]^q].
\end{equation}

Now, for $k=0,1,2,\dots$, one can compute, using the assumption that $y_0^{(1)}$ and $y_0^{(2)}$ are independent (and thus $\mathbb{E}[v_1(y_{i_1 - 1}^{(1)}) [v_2]^r(y_{i_2 - 1}^{(2)})] = 0$ for all $r, i_1, i_2$), 
\begin{align}
    \mathbb{E}[ [\phi_{k}^{(m)}]^p [\phi_{k}^{(m)}]^q] &= \sum_{i_1,i_2 =1}^k \sum_{r,s=1}^d [\Phi_{i_1}^{(m)}]^{pr} [\Phi_{i_2}^{(m)}]^{qs} \bigg(\frac{\mu_0^2}{m^{2/\alpha_1}} \mathbb{E}[v_1(y_{i_1 - 1}^{(1)}) v_1(y_{i_2 - 1}^{(1)})] \cdot [x_{i_1 - 1}^{(m)}]^r [x_{i_2 - 1}^{(m)}]^s \nonumber \\ 
    &\ \ \ \ \ + \frac{\sigma_0^2}{m^{2/\alpha_2}} \mathbb{E}[[v_2]^r(y_{i_1 - 1}^{(2)}) [v_2]^s(y_{i_2 - 1}^{(2)})]   \bigg), \label{eq_appcov}
\end{align}
where the $\Phi_{i}^{(m)} := \prod_{j=i}^{k-1} (I-\frac{1}{m} \nabla^2  \hat{\mathcal{R}}(\overline{x}_{j}^{(m)}))$, with the empty product taken to be the identity by convention.

The $l$th component of $\lambda_k^{(m)}$ can thus be written as: 
\begin{align}
    [\lambda_k^{(m)}]^l &= \frac{1}{m} \sum_{i=1}^k \sum_{j = 1}^d  [\Phi_{i}^{(m)}]^{lj} tr\left( C_{i-1}^{(m)} \nabla^2 [\nabla \hat{\mathcal{R}}(\overline{x}_{i-1}^{(m)}) ]^j  \right), \label{eq_term1}
\end{align}
where the $C_{i-1}^{(m)}$ are covariance matrices with the $(p,q)$-entry of $\mathbb{E}[ [\phi_{i-1}^{(m)}]^p [\phi_{i-1}^{(m)}]^q]
$ whose expression is given in Eq. \ref{eq_appcov}.

Lastly, we compute:
\begin{align}
    \mathbb{E} \left[ (\phi_k^{(m)})^T \nabla^2 \hat{\mathcal{R}}(\overline{x}_k^{(m)}) \phi_k^{(m)} \right] &=  \mathbb{E} \left[ \sum_{p,q = 1}^d [\phi_k^{(m)}]^p [\nabla^2 \hat{\mathcal{R}}(\overline{x}_k^{(m)})]^{pq} [\phi_k^{(m)}]^q \right] \\
    &=  \sum_{p,q = 1}^d \mathbb{E}[ [\phi_k^{(m)}]^p [\phi_k^{(m)}]^q ] \cdot [\nabla^2 \hat{\mathcal{R}}(\overline{x}_k^{(m)})]^{pq} \\
    &= tr\left( C_k^{(m)} \nabla^2 \hat{\mathcal{R}}(\overline{x}_k^{(m)}) \right). \label{eq_term2}
\end{align}

Substituting \eqref{eq_term1} and \eqref{eq_term2} into Eq. \eqref{eq_expansion}, we arrive at Eq. \eqref{thm3_eqn} in Theorem \ref{thm_implreg}.  
\end{proof}

\section{Additional Empirical Results and Details}
\label{app_sec:addexps}

In this section, we provide  additional empirical results, as well as the details and additional results for the experiments considered in the main paper, to strengthen the support for our theory. 

\subsection{Electrocardiogram (ECG) Classification}

We consider the Electrocardiogram (ECG) binary classification task that aims to discriminate between normal and abnormal heart beats of a patient that has severe congestive heart failure \cite{goldberger2000physiobank}. We use 500 sequences of
length 140 for training, and 4000 sequences for testing. We use a fully connected shallow neural network of width 32 with sigmoid activation and train for 1000 epochs, with the binary cross-entropy as the loss. We choose $\eta = 0.05$. For the MPGD and the Gaussian scheme, we inject the perturbations in the first 500 epochs instead, and stop injecting after that to allow the algorithm to converge. We use the  perturbation level $\mu = \sigma = 0.2$. The experimental setup is implemented in PyTorch, and all experiments are run on Google Colab.

Table \ref{tab:ecg} shows the average test accuracy (evaluated for 5 models that are trained with different seed values) for this task.
We see that overall MPGD improves the test accuracy when compared to the vanilla GD and Gaussian perturbations. In fact, the vanilla GD gets stuck in the loss landscape. While Gaussian perturbations can help to mobilize the GD iterates in the landscape, MPGD is more effective in steering the iterates to achieve higher test accuracy (see Figure \ref{fig_ecg}), which is the main goal of the learning task. This illustrates that adding the perturbations of MPGD can help improve the outcome of the optimization process in situations where the vanilla GD and  Gaussian perturbations fail to make meaningful progress. 

Note that the higher values of the perturbation levels used in the MPGD schemes lead to instabilities and fluctuations of large magnitudes in the earlier epochs.  These are necessary to boost the test accuracy; see the jump in the improvement of the test accuracy after epoch 500 in Figure \ref{fig_ecg}. The instabilities and fluctuations could be reduced if lower perturbation levels were used, but this would lower the test accuracy (which would still be higher than the baseline and the Gaussian results) obtained. Again, we emphasize that we are not going after competitive test performance here, but rather demonstrating the effectiveness of MPGD in improving the test performance in situations when other training schemes fail to make substantial progress in optimization.

\begin{figure}[!t] 
\centering
\includegraphics[width=0.48\textwidth]{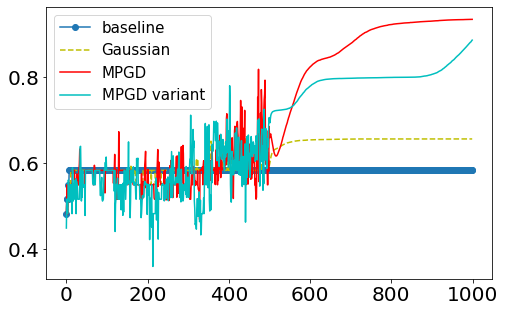}
\hspace{0.2cm}
\includegraphics[width=0.48\textwidth]{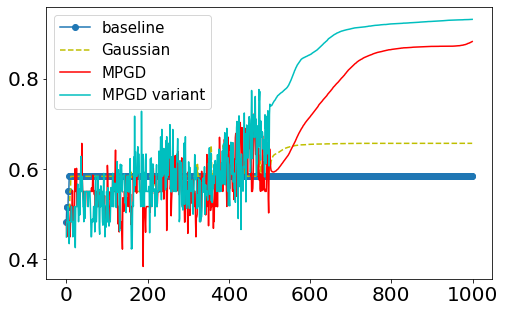}
\caption{Mean test accuracy obtained under various GD schemes over 1000 epochs for the ECG classification task. Here we choose the learning rate $\eta = 0.05$. MPGD refers to the scheme \eqref{eq_ouralg}, whereas  MPGD variant refers to  \eqref{DS_sym}, both using $\mu = \sigma = 0.2$. The MPGDs considered in the left plot use $\gamma = 0.55, \beta = 0$, whereas those in the right plot use $\gamma = 0.55, \beta = 0.5$. We see that applying the MPGD schemes in the first 500 epochs helps  to boost the test accuracy significantly when compared to the other two schemes.}
\label{fig_ecg}
\end{figure}

\begin{table}[!t]
	\caption{Shallow neural networks trained  on the ECG5000 Data Set for 1000 epochs.  The results in parenthesis are achieved with the variant of MPGD \eqref{DS_sym}. All the results are averaged over 5 models trained with different seed values. Here $\rm{std}$ denotes sample standard deviation. We use $\eta = 0.05$, and $\mu = \sigma = 0.2$ for the perturbation schemes. }
	\label{tab:ecg}
    \centering
        \scalebox{0.8}{
\begin{tabular}{@{}lccc@{}}
\toprule
\multicolumn{1}{c}{Scheme}   & mean test accuracy in \% & $\rm{std}$(test accuracy) in \% \\ \midrule
Baseline full batch GD                    &  58.38 &  0.0 \\
Gaussian                         & 65.68  &   16.34             \\
\midrule 
MPGD, $\gamma = 0.55$, $\beta = 0$      & 93.48 (88.65)  &  1.17 (9.16)        \\
MPGD, $\gamma = 0.6$, $\beta = 0$                       & {\bf 94.87} (90.52)   & 0.78 (6.69)         \\
MPGD, $\gamma = 0.65$, $\beta = 0$      &  92.61 (91.72)  &  0.99 (4.28)        \\
\midrule
MPGD, $\gamma = 0.55$, $\beta = 0.5$      & 88.22 (93.13)  &   13.99 (2.20)    \\
MPGD, $\gamma = 0.6$, $\beta = 0.5$                       & 94.27 (91.77)   &    1.06 (2.85)      \\
MPGD, $\gamma = 0.65$, $\beta = 0.5$      &   94.63 (83.87)  &  1.13 (23.69)        \\
\bottomrule
\end{tabular}
}
\end{table}

\subsection{Details and Additional Results for the Airfoil Self-Noise Prediction Task}

The experimental setup is implemented in PyTorch, and all experiments are run on 4x NVIDIA Tesla T4 GPUs with 16 GB VRAM belonging to an internal SLURM cluster. 
Table \ref{tab:stat_airfoil} reports the sample standard deviation for the results obtained in Table \ref{tab:airfoil}. 
Table \ref{tab:stat_airfoil_beta0} shows the corresponding results when different values of $\beta$ are used instead, illustrating the trade-offs induced by the selection of the stability parameter $\gamma$ and the skewness parameter $\beta$ for this particular setting. Here, we see that using the MPGD variant \eqref{DS_sym} with $\gamma = 0.6$  and $\beta = 0.5$ (see Table \ref{tab:airfoil}) leads to the lowest RMSE gap.

\begin{table}[!t]
	\caption{Statistics for the results obtained in Table \ref{tab:airfoil}. We report the sample standard deviation, denoted $\rm{std}$, of the test RMSE and RMSE gap. The values in parenthesis refer to the sample standard deviation results for the MPGD \eqref{DS_sym}. }
	\label{tab:stat_airfoil}
    \centering
    \scalebox{0.8}
    {
\begin{tabular}{@{}lcccc@{}}
\toprule
\multicolumn{1}{c}{Scheme} &   $\rm{std}$(test RMSE)  &  $\rm{std}$(RMSE gap)  \\ \midrule
Baseline     &   0.0595    & 0.0110    \\
Gaussian          &  0.0980 & 0.0136  \\
MPGD, $\gamma = 0.55$           &  0.1035 (0.0629)  &   0.0201 (0.0162)  \\
MPGD, $\gamma = 0.6$        & 0.0375 (0.1266)  & 0.0083 (0.0092)  \\
MPGD, $\gamma = 0.65$         & 0.0720 (0.0340)  &    0.0295 (0.0083) \\
MPGD, $\gamma = 0.7$    &  0.1091 (0.0094)  & 0.0162 (0.0076)    \\
\bottomrule 
\end{tabular}
}
\end{table}

\begin{table}[!t]
	\caption{Mean and standard deviation of RMSE gap under the setting used for obtaining the results in Table \ref{tab:airfoil} but with $\beta = 0$ (left) and $\beta = 0.25$ (right) for the MPGDs instead. The values in parenthesis refer to the corresponding results for the MPGD \eqref{DS_sym}. Here $\rm{std}$ denotes sample standard deviation, and $\eta, \mu$ and $\sigma$ are the same as the ones used for obtaining the results in Table \ref{tab:airfoil}. }
	\label{tab:stat_airfoil_beta0}
    \begin{center}
    \scalebox{0.8}
    {
    \begin{minipage}{0.05\textwidth}
    \hspace{-7.8cm}
\begin{tabular}{@{}lcccc@{}}
\toprule
\multicolumn{1}{c}{Scheme} &   mean  &  $\rm{std}$  \\ \midrule
MPGD, $\gamma = 0.55$           & 0.2308 (0.2325)   & 0.0091 (0.0107)   \\
MPGD, $\gamma = 0.6$        &  0.2274 (0.2339) &  0.0151 (0.0167)\\
MPGD, $\gamma = 0.65$         & {\bf 0.2273} (0.2373) & 0.0113 (0.0173) \\
MPGD, $\gamma = 0.7$    & 0.2362 (0.2333)  &  0.0092 (0.0093) \\
\bottomrule 
\end{tabular}
\end{minipage}
\hspace{0.15cm}
\begin{minipage}{0.05\textwidth}
\begin{tabular}{@{}lcccc@{}}
\toprule
\multicolumn{1}{c}{Scheme} &   mean  &  $\rm{std}$  \\ \midrule
MPGD, $\gamma = 0.55$           & 0.2264 ({\bf 0.2238})    &  0.0189 (0.0157) \\
MPGD, $\gamma = 0.6$        &  0.2267 (0.2247) & 0.0127 (0.0152)   \\
MPGD, $\gamma = 0.65$         & 0.2330 (0.2267) & 0.0207 (0.0045) \\
MPGD, $\gamma = 0.7$    & 0.2393 (0.2270)  & 0.0139 (0.0058)  \\
\bottomrule 
\end{tabular}
\end{minipage}
}
\end{center}
\end{table}

\subsection{Details for the CIFAR-10 Classification Task}

We use the implementation of ResNet18 and modify the implementation of the full-batch GD training provided at \url{https://github.com/JonasGeiping/fullbatchtraining} to set up MPGD and GD with Gaussian perturbations. Please refer to Section \ref{sect_empiricalresults} and \cite{geiping2021stochastic} for the relevant details.  The experimental setup is in PyTorch, and all experiments are run on an NVIDIA A100-SXM4 GPU with 40 GB VRAM belonging to an internal SLURM cluster. Table \ref{tab:cifar10-resnet18} reports the mean validation accuracies over runs of 5 models trained with different seed values, whereas Table \ref{tab:stat_cifar10} reports the sample standard deviation for the results in Table \ref{tab:cifar10-resnet18}. 

\begin{table}[!t]
	\caption{Statistics for the results obtained in Table \ref{tab:cifar10-resnet18}. We report the sample standard deviation (in \%), denoted $\rm{std}$, for the validation accuracy and the accuracy gap. The values in parenthesis refer to the sample standard deviation results for the MPGD variant \eqref{DS_sym}. }
	\label{tab:stat_cifar10}
    \centering
    \scalebox{0.8}
    {
\begin{tabular}{@{}lcccc@{}}
\toprule
\multicolumn{1}{c}{Scheme} &   $\rm{std}$(val. accuracy)  &  $\rm{std}$(accuracy gap)  \\ \midrule
Baseline    & 1.75        &  1.97    \\
Gaussian      & 1.13   & 0.53   \\
MPGD, $\gamma = 0.55$           & 2.52 (2.55)  & 2.45 (2.21)  \\
MPGD, $\gamma = 0.6$        & 1.52 (2.44) & 1.65 (2.00) \\
MPGD, $\gamma = 0.65$         & 1.09 (2.86)  & 1.22 (3.00) \\
MPGD, $\gamma = 0.7$    & 4.40 (1.57)  & 3.55 (1.10)    \\
\bottomrule 
\end{tabular}
}
\end{table}

\end{document}